\documentclass[10pt, conference, letterpaper]{IEEEtran}

\usepackage{amsmath}
\usepackage{amssymb}
\usepackage{mathtools}
\usepackage{amsthm}
\usepackage{bm}

\usepackage[utf8]{inputenc} 
\usepackage[T1]{fontenc}    
\usepackage{url}            
\usepackage{amsfonts}       
\usepackage{nicefrac}       
\usepackage{xcolor}         
\usepackage{caption}
\usepackage{multirow}
\usepackage{tablefootnote}
\usepackage{hyperref}

\usepackage{cases}
\usepackage{ulem}
\usepackage{xcolor}
\usepackage{enumitem}
\usepackage{adjustbox}
\usepackage{subcaption}
\usepackage{ragged2e}
\usepackage{pgfplots}
\pgfplotsset{compat=1.17}
\usepackage{thmtools,thm-restate}
\usepackage{blindtext}

\usepackage[numbers]{natbib}

\definecolor{nmgray}{RGB}{229,229,229}

\usepackage[capitalize,noabbrev]{cleveref}
\crefname{figure}{Fig.}{Figs.}

\usepackage{def-hao}

\definecolor{color1}{HTML}{105e8a}
\definecolor{color2}{HTML}{e99926}
\definecolor{color3}{HTML}{b82a0c}
\definecolor{color4}{HTML}{3e8a10}
\definecolor{color5}{HTML}{80037e}

\newif\iffinal
\finaltrue  
\iffinal
    \newcommand{\authnote}[2]{}
    \newcommand{\blue}[1]{}
    \newcommand{\thang}[1]{}
    \newcommand{\chicheng}[1]{}
    \newcommand{\hao}[1]{}
    \newcommand{\red}[1]{}
    \newcommand{\edit}[2]{}
    \newcommand{\ming}[1]{}
\else
    \newcommand{\authnote}[2]{{#1: #2}}
    \newcommand{\blue}[1]{{\color{blue}{#1}}}
    \newcommand{\thang}[1]{{\color{blue}\authnote{Thang}{#1}}}
    \newcommand{\chicheng}[1]{{\color{teal}\authnote{Chicheng}{#1}}}
    \newcommand{\hao}[1]{{\color{orange}\authnote{Hao}{#1}}}
    \newcommand{\red}[1]{{\color{red}{#1}}}
    \newcommand{\edit}[2]{{\color{color4}{\dashuline{#1}}}{\color{red}{#2}}}
    \newcommand{\ming}[1]{{\color{green}\authnote{Ming}{#1}}}
\fi

\newcommand{\prgreedy}{\ensuremath{\textsc{PR-Greedy}}\xspace}
\newcommand{\pretc}{\ensuremath{\textsc{PR-ETC}}\xspace}

\newcommand{\cumulativeregret}{\ensuremath{\textit{regret}}\xspace}

\newcommand{\sse}{\ensuremath{\text{SSE}}\xspace}
\renewcommand{\emph}[1]{\textit{#1}}

\ifCLASSINFOpdf

\else

\fi

\begin{document}

\title{Physics-Informed Parametric Bandits for Beam Alignment in mmWave Communications}

\author{ 
        Hao Qin*,
        Thang Duong*,
        Ming F. Li\IEEEmembership{Fellow,~IEEE,},
        Chicheng Zhang
        }

\markboth{Journal of \LaTeX\ Class Files,~Vol.~14, No.~8, August~2015}%
{Shell \MakeLowercase{\textit{et al.}}: Bare Demo of IEEEtran.cls for IEEE Journals}

\maketitle

\begin{abstract}
In millimeter wave (mmWave) communications, beam alignment and tracking are crucial to combat the significant path loss. 
As scanning the entire directional space is inefficient, designing an efficient and robust method to identify the optimal beam directions is essential. 
Since traditional bandit algorithms require a long time horizon to converge under large beam spaces, many existing works propose efficient bandit algorithms for beam alignment by relying on unimodality or multimodality assumptions on the reward function's structure.
However, such assumptions often do not hold (or cannot be strictly satisfied) in practice, which causes such algorithms to converge to choosing suboptimal beams.

In this work, we propose two physics-informed bandit algorithms \pretc and \prgreedy that exploit the sparse multipath property of mmWave channels  -  a generic but realistic assumption - which is connected to the Phase Retrieval Bandit problem. Our algorithms treat the parameters of each path as black boxes and maintain optimal estimates of them based on sampled historical rewards. \pretc starts with a random exploration phase and then commits to the optimal beam under the estimated reward function. \prgreedy performs such estimation in an online manner and chooses the best beam under current estimates. Our algorithms can also be easily adapted to beam tracking in the mobile setting. 
Through experiments using both the synthetic DeepMIMO dataset and the real-world DeepSense6G dataset, we demonstrate that both algorithms outperform existing approaches in a wide range of scenarios across diverse channel environments, showing their generalizability and robustness. 
\end{abstract}

\IEEEpeerreviewmaketitle

\section{Introduction}

Millimeter-wave (mmWave) communication is a key technology for 5G and FutureG wireless networks, enabling a wide range of applications that demand high bandwidth and low latency. Immediate applications in 5G include \citep{al2021overview}: Enhanced Mobile Broadband (eMBB), providing last-mile connectivity via Fixed Wireless Access (FWA), and wireless backhaul/fronthaul. It will also enable future use cases including Augmented and Virtual Reality (AR/VR), Vehicular Communications (V2X), and holographic telepresence. The main challenge is that mmWave signals, which operate at high frequencies (typically 30-300 GHz), suffer from significantly higher path loss compared to the lower-frequency signals. Thus, to counteract such signal degradation, mmWave systems rely on beamforming with high-gain antenna arrays, whose effectiveness hinges on the accurate alignment of the transmitter's and receiver's beams.   

Beam alignment in mmWave communications has remained a challenging problem. This is mainly because the beam patterns are so narrow that even a slight misalignment of the beam can result in a significant loss in signal strength, leading to a link failure. For example, \citet{nitsche2015steering} reported that with a 7-degree beam width, an 18-degree misalignment can reduce the link signal strength by 17 dB. 
Aligning the beam accurately with the direction that provides the highest gain is crucial to fully harness the potential of next-generation communication hardware. Another challenge comes from environmental/channel dynamics due to mobility or blockage.   When the user or base station is moving, even slightly, the optimal beam direction changes rapidly. Temporary blockages, such as a passing car or a person walking by, can disrupt the line-of-sight (LoS) path. The system must constantly track such dynamics and realign the beams in real-time to maintain the connection.  
 
Many beam alignment algorithms have been proposed in the past. For example, offline training-based methods  (such as beam scanning) have been proposed, which are adopted by existing standards including 802.11ad  \citep{nitsche2014ieee} and 5G NR \citep{parkvall2018nr}.  However, such methods are inefficient as they incur high overhead (linear to the number of beams, $K$) due to the large beam space, leading to reduced throughput. To fulfill real-time communication requirements, there is a critical need to develop an algorithm to select the best beam in a sample-efficient and online manner. In this paper, we investigate online beam alignment in a short-horizon setting.

Existing work has cast the online beam alignment problem as a Multi-Armed Bandit (MAB) problem \cite[e.g.,][]{hashemi2018efficient,wu2019fast}. 
The goal is for the learning agent to adaptively choose the arms (i.e., beams) in an online manner that yield the highest expected reward (i.e., received signal strength) through repeated interactions and feedback.
Many algorithms for the MAB problem have been proposed, such as Upper Confidence Bound (UCB)~\citep{lai1987adaptive,auer2002finite} and Explore-Then-Commit (ETC)~\citep{langford2007epoch,lattimore2020bandit}. However, they require a long time to converge. 
To improve sample efficiency, \citet{yu2011unimodal,cutkosky2023blackbox} studies the setting where the expected reward is a unimodal function of the arm (i.e., it has only one peak, corresponding to the LoS/dominant path direction), a framework adopted by several prior works in beam alignment~\citep{hashemi2018efficient,ghosh2024ub3}. Theoretically, unimodal bandit problems achieve regret guarantees independent of the number of arms. 
Subsequent works~\citep{saber2024bandits} extend this to the setting where the reward is a multimodal function of the arm, however, they typically have restrictive assumptions about a known (or a bound) number of peaks. 

Nevertheless, existing MAB approaches and their extensions are insufficient for mmWave beam alignment problems, due to the following:

(1) The real-world mmWave channel induces reward functions that are often not unimodal.  Even with only a Line-of-Sight (LoS) path, because of the antenna sidelobes, the reward function contains many local peaks -- see \Cref{fig:reward_func}. In addition, the mmWave signal propagation is generally modeled using a geometric-based statistical model which includes multiple reflection paths (or clusters) \citep{rappaport2017overview,gustafson2013mm}, which is more complex than the LoS channel model that unimodal bandits rely on. 

In reality, the channel parameters, such as the number of paths, reflection coefficients, and the path-loss exponents, all depend on the environment, which cannot be exactly known in advance.   
    Previous methods that assume unimodality or multimodality \citep{yu2011unimodal,hashemi2018efficient,cutkosky2023blackbox} fail to generalize under such realistic channel settings.

(2) Real-world applications of beam alignment (such as V2X) oftentimes require low latency for real-time applications, and require selecting good beams in a short time horizon, such as under 100 ms, which implies a few tens of time steps (or subframes in 5G NR)~\citep{hassanieh2018fast,mazaheri2019millimeter}. 

\begin{figure}[t]
    \centering
    \includegraphics[width=0.9\linewidth]{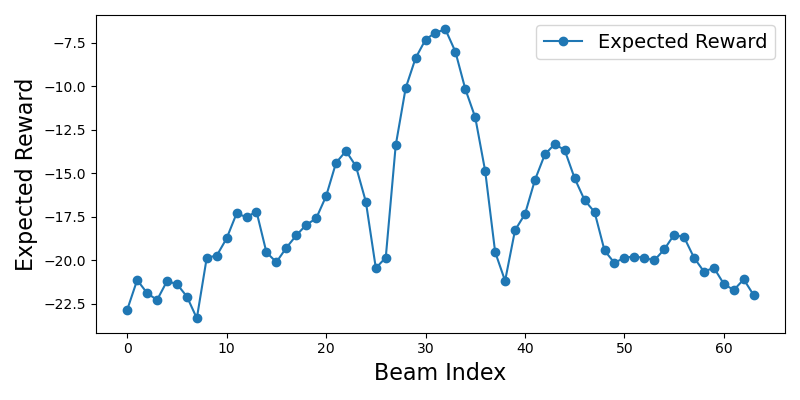}
    \caption{The expected reward function of a bandit instance from the DeepSense 6G dataset (scenario 17). Such reward function is not unimodal, and thus algorithms that rely on unimodality assumptions may converge to choosing a suboptimal beam.} 
    \label{fig:reward_func}
\end{figure}
  
In this paper, inspired by works on parametric bandits~\citep{filippi2010parametric}, we propose physics-informed parametric bandit algorithms, \pretc and \prgreedy, for beam selection by exploiting the underlying reward function structure similar to the Phase Retrieval (PR) Bandit problem. 
We found that the electromagnetic wave propagation model commonly used in far-field  mmWave channel models~\citep{samimi2014characterization}
can be viewed as a phase-retrieval model and solved using a similar approach for parametric bandits in \citep{filippi2010parametric}, while taking into consideration the properties of phase-array antennas.
Empirically, we demonstrate that leveraging the unimodal or multimodal property is not as efficient and robust as leveraging the underlying beam pattern and wave propagation structure.

The contributions of this paper are summarized as follows:
\begin{itemize}
    \item   We propose two physics-informed parametric-bandit algorithms for mmWave beam alignment:   \prgreedy and  \pretc, that leverage the sparse multipath channel model in mmWave and knowledge of the codebook/beam patterns.  While \prgreedy achieves a lower empirical regret, \pretc is a more computationally efficient approximation of \prgreedy that slightly trades off performance for lower computational latency.

  \item We theoretically analyze the regret of \pretc, and show that under reasonable assumptions on the properties of the reward function, it can be upper bounded by $O(k^{1/3} T^{2/3})$ (where $k$ is the number of propagation paths, $T$ is the time horizon), which does not depend on the number of beams. We also provide a preliminary analysis on the regret guarantee of PR-GREEDY.

    \item  We evaluate our algorithms using two large-scale synthetic and real-world datasets: DeepMIMO, and DeepSense6G. Results demonstrate the adaptability and robustness of our algorithms across 4,952 bandit instances from the DeepMIMO simulated environment and 12 bandit instances from DeepSense6G, using a default hyperparameter, which significantly outperform several baselines in terms of sample efficiency. 
    Results under mobile scenarios in DeepSense6G show that our algorithms can quickly adapt to mobility and blockage using a periodical reset strategy. 
    We also show that \prgreedy and \pretc is robust under model misspecification, without imposing exact assumptions on the underlying structure of the channel/reward function (such as the number of paths, the number of peaks, etc.). 
\end{itemize}

\section{Related work}

For offline beam training, although previous works proposed improved offline algorithms with sub-linear complexity (e.g., hierarchical beam search \citep{hassanieh2018fast} or compressive sensing \citep{hekaton} by exploiting the sparse nature of the mmWave channel, they require advanced hardware (such as multi-beam antenna arrays, or require hybrid beamforming), and are not very robust nor sample-efficient under noisy channels. In contrast, our algorithms assume analog beamforming and only need basic hardware - a phase shift antenna array.

For online beam alignment in mmWave communications, many prior works leverage the unique structure of the reward function to improve the online learning regret, such unimodality~\citep{hashemi2018efficient,combes2014unimodal,ghosh2024ub3}, multimodality~\citep{saber2024bandits}, and weak-Lipschitzness of the reward function~\citep{wu2019fast}, respectively,
while~\citet{zhao2022hierarchical,zhang2021mmwave} utilizes hierarchical organization of codebooks.   However, the real-world mmWave channel's RSS is often not strictly unimodal or multi-modal, and these algorithms are not robust under model mis-specification, as we show later in our results.   

Unimodal or multi-modal bandits have their roots in machine learning literature. 
\citet{yu2011unimodal,cutkosky2023blackbox} focus on bandit with a unimodal structure of the reward function. \citet{yu2011unimodal} proposes the Line Search Elimination Algorithm (LSE) to efficiently select the action with the highest expected reward by leveraging the unimodal property in general, not specifically designed for beam alignment.
The noiseless and noisy 1-d convex bandit algorithms in \citep{agarwal2011stochastic} can also be used to exploit the unimodal structure. \citet{cutkosky2023blackbox} proposes a more sample-efficient approach, but it maintains expected reward estimates with a large collection of arms, which limits its application in real-time deployment. Recent work of \citet{saber2024bandits} extends this line of work to handle reward function with multimodal rewards; thus far, the analysis has been focusing on the long-horizon asymptotic regime. 

For the nonstationary reward distribution setting, \citet{liu2017changedetectionbasedframeworkpiecewisestationary,krunz2023online} tackle this using a nonstationary multi-armed bandit formulation by focusing on action selection using the more recent data. In light of the observation that some contextual information may be available in beam selection (e.g. moving direction of the vehicle UE is at), \citet{sim2018online} solve this using a contextual bandit formulation. \citet{jeong2020online} formulate joint beam selection and beam pattern optimization as a two-level deep reinforcement learning problem, showing its effectiveness in beamtracking problem. \citet{zhang2021mmwave} solves the beam pattern design problem using deep reinforcement learning and show the learned codebooks can better serve mobile users. 
Overall, we view existing techniques to extend to the mobile setting as orthogonal to the contribution of our work and can be combined with our methods.

\section{Problem formulation}
\label{sec:prelims}

\subsection{System Model}

We study the beam alignment or tracking problem, where a base station (BS)   communicates with a user equipment (UE) in the mmWave band (e.g., 28GHz in 5G NR  FR2, or  60 GHz in 802.11ad). 
The UE can be stationary or mobile. The BS is equipped with an  antenna array of $N$ elements and a single RF chain (the UE can also be equipped with an array, but usually has a smaller number of antenna elements).
The BS performs transmit or receive beam steering (corresponding to downlink or uplink transmissions) by choosing a steering vector
$\fbf \in \CC^{N}$
from a known codebook $\Fcal = \cbr{\fbf_a: a \in [K]}$, to maximize the communication quality (or received signal strength, RSS). 
\footnote{We assume that the UE adopts an omni-directional or quasi-omni beam pattern. Our method can be easily generalized to joint beam selection on both ends. }

We consider a far-field channel, which applies when the distance between the transmitter and receiver is significantly larger than the Fraunhofer distance \citep{rappaport2017overview} (e.g., in outdoor settings, the UE is far away from the BS). For mmWave communication, a geometric channel model is often used \citep{samimi2014characterization,rappaport2017overview}, which contains a limited number of propagation paths or scatters.  Such paths can be the Line-of-Sight (LoS) path or other paths induced by reflectors. 
See \Cref{fig:system-model} for an illustration. 

\begin{figure}[t]
    \centering
    \includegraphics[width=1\linewidth]{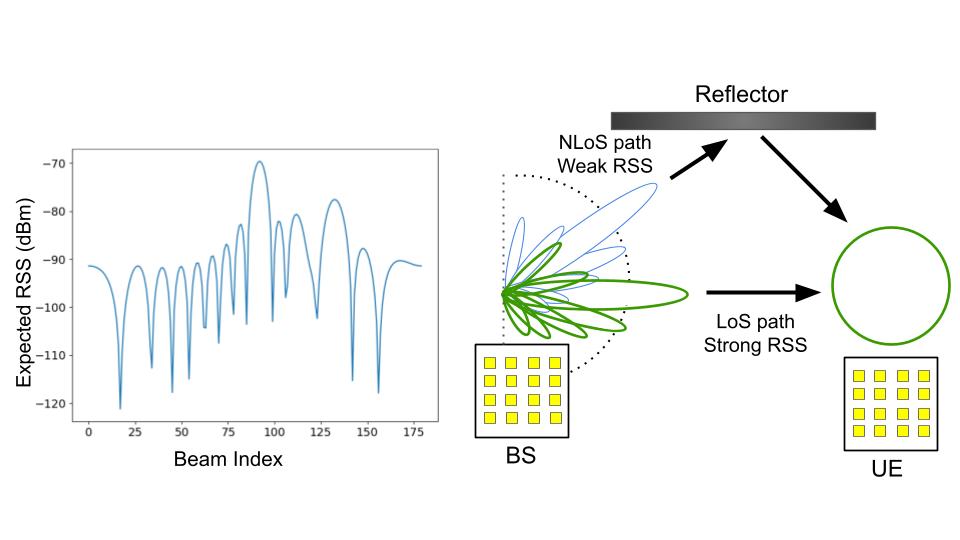}
    \caption{An overview of our system model.
    }
    \label{fig:system-model}
\end{figure}

Physically, each steering vector corresponds to a different pattern of phase shifts on the $N$ antenna elements. Given a steering vector $\fbf_a$, its \emph{beam pattern} is a function that maps every possible angle of UE $\theta$ to received signal amplitude through line of sight at unit distance: 
\[
h_a(\theta) = \fbf_a^\top v(\theta),
\]
where $v(\theta) \in \CC^N$ 
represents the array response vector with AoA/AoD equal to $\theta$.

We model the channel between the BS and UE, when the BS chooses the steering vector $\fbf_a$, as a combination of channels induced by $k$ paths: 
\[
h_a(\bm \theta, \bm \beta)
= 
\sum_{i=1}^k \beta_i h_a(\theta_i^*).
\]
Here, $\bm \theta^*:= (\theta_1^*, \ldots, \theta_k^*)$ denotes the AoA or AoD of each path, each chosen from a discrete set $\Theta \subset [0, 2\pi]$; 
$\bm \beta := (\beta_1^*, \ldots, \beta_k^*)$ denotes the complex gain of each path (including the path loss and possible reflection gain), with each $\beta_i^*$ lying in a grid $B \subset \CC$. 
In practice, a small value of $k$ models the channel well, i.e., $k\leq5$.

The received signal strength (in dBm) with antenna configuration $\fbf_a$ is expressed as \Cref{eqn:noisy-reward}:   
\begin{align} 
    r
    =&
    30 + 10 \log_{10} \rbr{\abs{ h_a(\bm \theta^*, \bm \beta^*)}^2} +\eta, 
        \label{eqn:noisy-reward}
\end{align}
where $\eta$ is a zero-mean Gaussian noise.

We define the expected signal strength as a function of antenna configuration $\fbf_a$ and channel with parameters $(\bm \theta^*, \bm \beta^*)$ as: 
\begin{align} 
    R(\fbf_a, \bm \theta^*, \bm \beta^*)
    \coloneqq&
    30 + 10 \log_{10} \rbr{\abs{h_a(\bm \theta^*, \bm \beta^*)}^2}
        \label{eqn:noiseless-reward}
\end{align}

Initially, the base station does not have any knowledge about the channel parameters $(\bm \theta^*, \bm \beta^*)$. To choose the best steering vector, we aim at learning these  
$2k$ parameters. This makes our method applicable to different path loss exponents and carrier frequencies. In addition, as we will see, our method treats each path as a black box and doesn't need to know the details of each path; so it works as long as the beam pattern for each antenna configuration is known. 

\textit{Example: uniform linear array}. 
Suppose the base station's antenna is a uniform linear array (ULA) with $N = 2\bar{N} + 1$, with antenna elements indexed as $n \in \cbr{-\bar{N}, \ldots, 0, \ldots, \bar{N}}$. 
In this case, the array response vector
$v(\theta)$ has the following explicit form~\citep[e.g.,][]{liu2023near}: 
\[
v(\theta)
= 
\del{e^{-j\frac{2\pi}{\lambda} n d \cos\theta} }_{n=-\tilde{N}}^{\tilde{N}},
\]
where $\lambda$ denotes the wavelength of the carrier, $j$ is the imaginary unit, $n \in \cbr{-\bar{N}, \ldots, 0, \ldots, \bar{N}}$ is the index of the antenna element (with $0$ representing the central element). For ULAs, a popular choice of the steering vectors $\fbf_a$ is of the Discrete Fourier Transform form: 
$
    \fbf_a \coloneqq \del{e^{j \frac{2\pi}{\lambda} (nd \cos{\pi\frac{a}{K}})}}_{n=-\bar{N}}^{\bar{N}}
$; as we can see, $\fbf_a$ with $a = K \theta$ maximizes $|h_a(\theta)| = |\fbf_a^\top v(\theta)|$, representing the beam that has the highest RSS when the AoD / AoA is equal to $\theta$. 

\subsection{Beam Alignment as a Parametric Bandit Problem}

\textit{Learning protocol}: 
For each time step $t = 1,\ldots, T$, the BS chooses beam index $a_t \in [K]$ and  transmits a probe/data frame tagged with this index,  and 
receives signal strength feedback  
$r_t = R(\fbf_{a_t}, \bm \theta_t^*, \bm \beta_t^*) + \eta_t$, where $\eta_t \sim N(0, \sigma^2)$, and $\bm \theta_t^*, \bm \beta_t^*$ is the channel parameter at time step $t$.
In the \emph{stationary reward distribution setting}, we have $(\bm \theta_1^*, \bm \beta_1^*) = \ldots = (\bm \theta_T^*, \bm \beta_T^*) = (\bm \theta^*, \bm \beta^*)$. In contrast, the \emph{changing reward distribution setting} allows the $(\bm \theta_t^*, \bm \beta_t^*)$'s to be time-varying, which captures settings that the transmitter or the receiver is moving or experiences blockage by moving objects. 

In the stationary reward distribution setting, our goal is to minimize the \cumulativeregret, which is the cumulative gap of signal strength between the perfectly aligned and the steering vector selected by the algorithm up to time step $T$, defined as: 
\begin{align*}
    \Regret_T
    \coloneqq& \sum_{t=1}^T \max_{a \in [K]} R(\fbf_a, \bm \theta^*, \bm \beta^*) - R(\fbf_{a_t}, \bm \theta^*, \bm \beta^*) 
\end{align*}
Similarly, in the changing reward distribution setting our goal is to minimize the \emph{dynamic regret}, defined as: 
\begin{align*}
    \mathrm{DynamicRegret}_T
    \coloneqq& \sum_{t=1}^T \max_{a \in [K]} R(\fbf_a, \bm \theta_t^*, \bm \beta_t^*) - R(\fbf_{a_t}, \bm \theta_t^*, \bm \beta_t^*) 
\end{align*}

\section{Algorithms}

\subsection{Algorithms for the Stationary Reward Distribution Setting}

For the static setting, we introduce two versions of physics-informed parametric bandit algorithms: Phase Retrieval Explore-then-Commit (\pretc, \Cref{alg:glm-etc}) and Phase Retrieval Greedy (\prgreedy, \Cref{alg:glm-greedy}), which correspond to Explore-then-commit (ETC) and Greedy strategies.
\pretc follows an ETC~\citep{lattimore2020bandit} style via selecting beams uniformly at random for $M$ time steps and collecting reward feedback to estimate $\bm \beta^*, \bm \theta^*$.
By choosing $M$ appropriately, it balances exploration and exploitation. 
\prgreedy focuses on exploitation by choosing beams greedily according to the current RSS estimate 
and updating its parameters at every time step.

Note that our algorithms are compatible with mmWave communication standards. For example, in our \pretc algorithm, the probing frames can be implemented by Synchronization Signal Block (SSB) in 5G NR which does not contain actual data and are very short (\textasciitilde 1 ms each). For \prgreedy, the beam index can be included in the header of a data subframe. RSS/reward feedback $r_t$ (defined in \Cref{eqn:noisy-reward}) can be carried in an ACK packet in that time step.

Both algorithms use MLE to produce estimates $(\hat{\bm \beta},  \hat{\bm \theta})$ on the channel parameters, which we provide details below.

\subsubsection{Channel Estimation using Maximum Likelihood Estimate (MLE)}

Suppose that we have a dataset of $m$ action-reward pairs $S_{m} \coloneqq \cbr{a_t, r_t}_{t=1}^{m}$.
We aim to find the estimates $\hat{\bm \beta}_m, \hat{\bm \theta}_m$ by maximizing the likelihood function $\Lcal$, which, according to Eq.~\eqref{eqn:noisy-reward}, is: 
\begin{align*}
    \Lcal\del{\bm \beta, \bm \theta \mid S_{m}} \coloneqq \prod_{t=1}^{m} \frac{1}{\sqrt{2 \pi \sigma^2}} \expto{-\frac{(r_t - R(\fbf_{a_t}, \bm \beta, \bm \theta)^2}{2\sigma^2}}.
\end{align*}

Note that maximizing the likelihood function $\Lcal\del{\bm \beta, \bm \theta \mid S_{m}}$ over the parameter space is equivalent to minimizing the square loss:
\begin{align}
    & \argmax_{\bm \theta, \bm \beta} \Lcal\del{\bm \beta, \bm \theta \mid S_{m}} 
    =
    \argmin_{\bm \theta, \bm \beta} \sum_{t=1}^{m} \del{r_t - R(\fbf_{a_t}, \bm \beta, \bm \theta)}^2
        \label{eqn:optimization-target}
\end{align}

\begin{algorithm}[t]
\begin{algorithmic} 
\STATE \textbf{Input:} Number of paths $k$, number of beams $K$, codebook matrix $F$, 
exploration length $M$

\FOR{$t=1,2,\cdots,M$}
    \STATE Choose beam $a_t \sim U([K])$, where $U$ denotes the uniform distribution
    \STATE Receive reward \( r_t \).
\ENDFOR

\STATE Estimate parameters $\hat{\bm \beta}_M, \hat{\bm \theta}_M$ by solving \Cref{eqn:optimization-target}.

\FOR{$t=M+1,\ldots,T$}

\STATE Choose beam $a_t := \argmax_{a \in [K]} R(\fbf_a, \hat{\bm \beta}_M, \hat{\bm \theta}_M)$
\ENDFOR
\end{algorithmic}
\caption{
\pretc
} 
\label{alg:glm-etc}
\end{algorithm}

\begin{algorithm}[t]
\begin{algorithmic} 
\STATE \textbf{Input:} Number of paths $k$, number of beams $K$, codebook matrix $F$, initial parameter $\hat{\bm \beta}_0$ and $\hat{\bm \theta}_0$,
\FOR{$t=1,2,\cdots,T$}
    \STATE Choose beam $a_t := \argmax_{a \in [K]} R(\fbf_a, \hat{\bm \beta}_{t-1}, \hat{\bm \theta}_{t-1} )$
    \STATE Receive reward \( r_t \).
    \STATE Estimate parameters $\hat{\bm \beta}_t, \hat{\bm \theta}_t$ by solving \Cref{eqn:optimization-target}.
\ENDFOR
\end{algorithmic}
\caption{
\prgreedy
} 
\label{alg:glm-greedy}
\end{algorithm}

Note that \pretc only solves this optimization problem once, which is at the last time step $M$ of the exploration phase, while \prgreedy solves the optimization problem $T$ times.
Thus, the computational cost of \pretc can be much smaller than \prgreedy.

\subsubsection{Regret guarantee for \pretc} 

We establish regret guarantees of \pretc under three key assumptions. Below we use the convention that $\log$ function applies logarithm to a vector in an entrywise manner.

\begin{assum}[Bounded reward function]
For all $\bm \theta \in B^k$ and $\bm \beta \in \Theta^k$, and steering vector $\fbf_a$,  $\abs{R(\fbf_a, \bm \theta^*, \bm \beta^*)} \leq R_{\max}$.
\label{assum:bounded}
\end{assum}

\cref{assum:bounded} states that all reward functions under consideration are bounded; in practice, $R_{\max}$ can be chosen as a constant, e.g. 70 (dBm). 

\begin{assum} \label{assum:quadratic-lower-bound}
There exist constants $C_1$ and $C_2$ such that the following inequality holds: $\forall
        \bm \beta \in B^k, 
        \bm \theta \in \Theta^k,
        \theta_1 < \ldots, \theta_k$
    \begin{align}
        & C_1 \| \bm \theta^* - \bm \theta \|^2 + C_2 \| \log \bm \beta^* - \log \bm \beta \|^2 \nonumber \\
        \leq & 
        \EE_{a \sim \text{Unif}(\Acal)} \del{R(\fbf_a, \bm \theta^*, \bm \beta^*) - R(\fbf_a, \bm \theta, \bm \beta)}^2, 
        \label{eqn:RHS-assumption1}
    \end{align}
\end{assum}

\Cref{assum:quadratic-lower-bound} informally states that the landscape of the mean square error of predicted expected reward of all beams
as a function of $(\bm \theta, \log \bm \beta)$ is strongly convex. Indeed, we verify this using the DeepMIMO dataset with base station $1$ and user location index $(0, 298)$ with the number of paths $k=1$ by setting $C_1 = 6.2$, $C_2 = 331.1$; see \Cref{fig:loss-surface}.\footnote{We choose $k=1$ since the relevant functions can be visualized in 3D.}

\begin{figure}[t]
    \centering
    \includegraphics[width=0.75\linewidth]{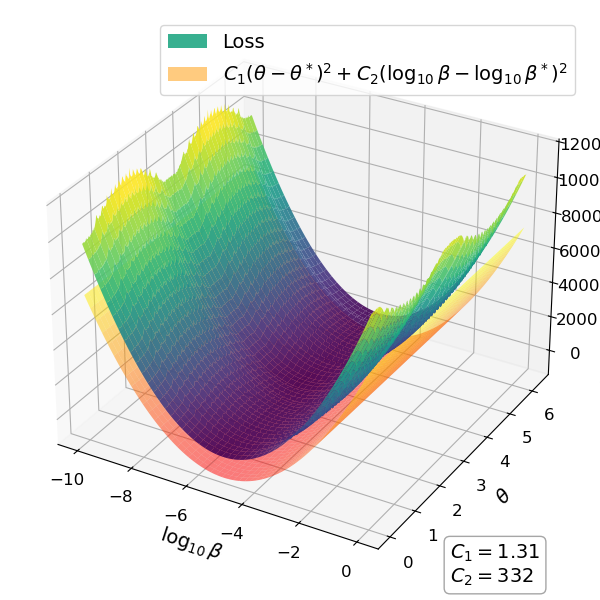}
    \caption{
    The right-hand side (in green) vs the left-hand side (in orange) of Eq.~\eqref{eqn:RHS-assumption1}, verified on the DeepMIMO scenario 4 Base Station 1 User Equipment $(0, 843)$. 
    } 
    \label{fig:loss-surface}
\end{figure}

\begin{figure}[t]
    \centering
    \includegraphics[width=0.75\linewidth]{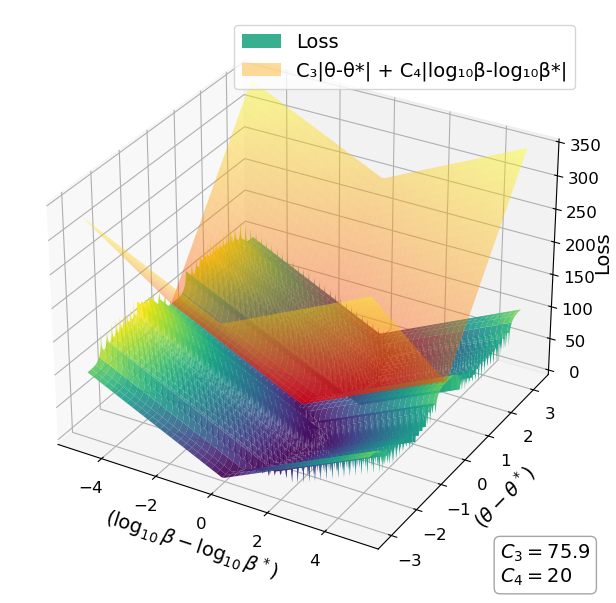}
    \caption{
    The left-hand side (in green) vs the right-hand side (in orange) of Eq.~\eqref{eqn:rhs-assumption2},
    verified on the DeepMIMO scenario 4 Base Station 1 User Equipment $(0, 843)$, with a fixed beam $a = 0$.
    }
    \label{fig:assumption-2-loss-surface}
\end{figure}

\begin{assum}[Local Lipschitz of the reward function]
\label{assum:lip}
There exist constants $C_3$ and $C_4$ such that the following inequality holds: $\forall \bm \beta \in B^k, \bm \theta \in \Theta^k, \forall a \in [K]$, 
    \begin{align}
        & \abs{R(\fbf_a, \bm \theta^*, \bm \beta^*) - R(\fbf_a, \bm \theta, \bm \beta)} 
        \nonumber
        \\
        \leq &
        C_3 \| \bm \theta^* - \bm  \theta \| + C_4 \| \log \bm \beta^* - \log \bm \beta \|,
        \label{eqn:rhs-assumption2}
    \end{align}
\end{assum}

Assumption~\ref{assum:lip} states that the reward of any beam changes smoothly as the channel parameters $\bm \beta, \bm \theta$ deviate from the ground truth parameter $\bm \beta^*, \bm \theta^*$.
In \Cref{fig:assumption-2-loss-surface}, we verify this using the DeepMIMO dataset with base station $1$ and user location index $(0, 298)$, and beam index $a = 0$ with $k=1$ by setting $C_3 = 21.05$, $C_4 = 37.98$.

Under the above three assumptions, we prove the following regret guarantee for \pretc: 
\begin{theorem}
\label{thm:reg-etc}
Suppose Assumptions~\ref{assum:bounded},~\ref{assum:quadratic-lower-bound}, and~\ref{assum:lip} hold. 

The regret of \pretc satisfies:
\[
\Regret_T \leq O\del{ M R_{\max} + T \sqrt{ \frac{k\sigma^2 R_{\max}^2 (\log |B| + \log|\Theta| )}{M} } }. 
\]
Furthermore, with $M = T^{2/3} \del{k\sigma^2 (\log |B| + \log|\Theta|)}^{1/3}$, 
\[
\Regret_T \leq O\del{ R_{\max} T^{2/3} \del{k\sigma^2 (\log |B| + \log|\Theta|)}^{1/3}  }. 
\]
where $O(\cdot)$ hides constants that depends on $C_i, i =1,\ldots,4$.
\end{theorem}
\begin{remark}
Recall that $|B|$ and $|\Theta|$ are the spaces where channel parameters $\beta_i$ and $\theta_i$ lie; Therefore, $\log|B|$ and $\log|\Theta|$ are roughly the number of bits needed to encode $\beta_i$ and $\theta_i$, which can be regarded as constants.~\footnote{In our experiments, both $\log|B|$ and $\log|\Theta|$ are no more than 7.}
\end{remark}
\begin{remark}
Our algorithms only have a few hyperparameters, making them suitable for deployment with minimal hyperparameter tuning: 
\prgreedy only requires $k$; while \pretc only requires $k$ and $M$.
They do not require the knowledge of $C_i, i = 1,\ldots,4$, for example. 
\end{remark}

We now provide a preliminary analysis on the regret guarantee of \prgreedy. Following recent  analysis of  greedy algorithm for structured bandits~\cite{slivkins2025greedy}, we make the following assumption on our reward function class:  
\begin{assum}[$\gamma$-self-identifiability]
There exists some $\gamma > 0$, such that the class of reward functions $\cbr{ R(\fbf_a, \bm \theta, \bm \beta): \bm \theta \in \Theta^k, \bm \beta \in B^k }$ is $\gamma$-self-identifiable w.r.t. ground truth parameter $(\bm \theta^*, \bm \beta^*)$; i.e., if for any suboptimal beam $a \in [K]$, and any $(\bm \theta, \bm \beta) \in \Theta^k \times B^k$ such that $a = \argmax_{b \in [K]} R(\fbf_b, \bm \theta, \bm \beta)$, 
$|R(\fbf_a, \bm \theta, \bm \beta) - R(\fbf_a, \bm \theta^*, \bm \beta^*)| \geq \gamma$.
\label{assum:self-identifiability}
\end{assum}

The $\gamma$-self-identifiability assumption requires that any incorrect estimates of the channel parameter $(\bm \theta, \bm \beta)$ can be invalidated by taking its greedy beam and observing their rewards. Without the self-identifiability assumption, \prgreedy may get stuck at taking a suboptimal action and maintaining incorrect channel parameter estimates, resulting in linear regret. Empirically, we verified that in the setting of $B = \sbr{e^{-3+0.25i}, i \in {0, 1, \ldots, 10}}, \Theta = \cbr{2i, i \in \cbr{0, 1, \ldots, 180}}$, $R(\fbf_a, \bm \theta, \bm \beta)$ satisfied \Cref{assum:self-identifiability} with $\gamma=7\times10^{-3}$ and in the setting of $B = \sbr{e^{-3+0.5i}, i \in {0, 1, \ldots, 4}}, \Theta = \cbr{2i, i \in \cbr{0, 1, \ldots, 180}}$, $\gamma$ becomes $0.071$ for the DeepMIMO dataset. 

\begin{theorem} \label{thm:reg-pr-greedy}
Suppose our class of reward functions is $\gamma$-self-identifiable with respect to ground truth parameter $(\bm \theta^*, \bm \beta^*)$. Then, when \prgreedy is run, with probability $1-\delta$, each suboptimal beam will be chosen at most 
$\tilde{O}\rbr{ \frac{k}{\gamma^2} (\ln|B| + \ln|\Theta|) }$ times.
\end{theorem}

The proof of the above theorem can be found in \Cref{app:proof-reg-pr-greedy}.
While this improves over the original analysis of~\cite{slivkins2025greedy} by replacing $K$ with $k$ via utilizing the structure of our reward function class, we believe that this bound may still be quite pessimistic since the total number of times suboptimal beams are selected can still be $\Omega(K)$. 
We conjecture that it is possible to bound the \emph{total number of suboptimal beams taken} by \prgreedy by some quantity independent of $K$, by exploiting the more fine-grained structure of our reward function class.

As we can see from Theorem~\ref{thm:reg-etc}, 
when the channel has at most $k$ paths,
\pretc has a regret of $\tilde{O}\del{ T^{2/3} k^{1/3} }$, where $\tilde{O}$ hides the logarithmic terms.
This is significantly better than the standard Explore-then-Commit (ETC) algorithm~\citep{lattimore2020bandit}, which only achieves $\tilde{O}(T^{2/3} K^{1/3})$ since we avoid the dependence on the size of the action space $K$, which is often much larger than $k$ in our application.
In comparison, UCB~\citep{auer2002finite} achieves a regret of $\tilde{O}(\sqrt{KT})$, which can be even vacuous since the number of beams $K$ may be as large as $T$ in our application.
LSE~\citep{yu2011unimodal} assumes that the reward function is unimodal, BISECTION~\citep{agarwal2011stochastic} assumes that the reward function is convex and Lipschitz continuous, which are restrictive compared to our assumptions since their algorithm is sensitive to the choice of Lipschitz constant such as LSE~\citep{yu2011unimodal} needs to know the Lipschitz constant of the reward function to achieve sublinear regret; specifically, underspecifying the Lipschitz constant may result in convergence to a suboptimal beam. In contrast, our algorithm only has an exploration length parameter $M$, which our algorithm is robust to.

\textit{Extension to the Changing Reward Distribution Setting}: We extend \pretc and \prgreedy to handle the mobile setting by running an instance of it from scratch every $\tau$ time steps - see Algorithm~\ref{alg:restart}. Intuitively, within any interval of $\tau$ time steps, the environment is much more stationary than the overall nonstationary problem with $T$ time steps, and thus we expect the algorithms in the static setting to perform reasonably well. Such restart strategies has recently been analyzed in linear bandit~\citep{zhao2020simple} and beam tracking applications~\citep{deng2022interference}. 
We leave the investigation of a combination of our physics-informed algorithm with adaptive restart~\cite[e.g.,][]{liu2018change} as an interesting open question.

\begin{algorithm}[ht]
\begin{algorithmic} 
\STATE \textbf{Input:} Number of paths $k$, number of beams $K$, codebook matrix $F$, restart time threshold $\tau$, a base bandit algorithm $\Acal$
\FOR{$n = 1,2, \cdots, \lceil \frac{N}{\tau} \rceil$}
    \STATE Run an new instance of the base algorithm $\Acal$ for $\tau$ time steps
\ENDFOR
\end{algorithmic}
\caption{
Periodic-$\Acal$
} 
\label{alg:restart}
\end{algorithm}

\section{Performance Evaluation}

In this section, we evaluate \pretc and \prgreedy using two datasets in the mmWave band: a synthetic dataset based on ray-tracing simulation, DeepMIMO~\citep{alkhateeb2019deepmimo}, and a real-world measurement dataset, DeepSense6G~\citep{DeepSense}. These datasets include a variety of realistic channel environments with complex reward functions, and both static and mobile scenarios. 

\subsection{Evaluation Metrics and Baselines} \label{sec:evaluation-metric}

Since we evaluate across many tasks, to make the regrets of all tasks comparable, we scale the regret in all tasks to a common range.
The normalized regret of an algorithm is defined as the ratio of its regret and that of the naive algorithm that chooses beams uniformly at random at every time step, denoted as N-Regret$_T$:

\begin{align}
    & \text{N-}\Regret_T :=  \nonumber \\
    & 
    \frac{\Regret_T}{T \cdot (\max_{a \in [K]} R^*(\fbf_a, \bm \theta^*, \bm \beta^*) - \EE_{a \sim U([K])}[R(\fbf_a, \bm \theta^*, \bm \beta^*)]) }
        \label{eqn:normalized-static-regret}
\end{align}

The normalized regret for any reasonable algorithm is most likely in $[0, 1]$, whereas the algorithm that always chooses the optimal beam has zero normalized regret. 

We compare our \pretc and \prgreedy algorithms with several baselines: 
\begin{itemize}
\item Upper Confidence Bound (UCB)~\citep{auer2002finite}: a basic version of the algorithm, where the confidence bound is given by $\sqrt{2\ln(T)/N_{t, a}}$ with $N_{t, a}$ representing the number of times beam $a$ has been taken up to time step $t$. 

\item Line Search Elimination (LSE)~\citep{yu2011unimodal}: an algorithm that optimizes the reward function by exploiting its unimodal property; this has also recently been evaluated in beam alignment~\citep{zhao2022hierarchical}.

\item BISECTION~\citep{agarwal2011stochastic}: a noisy bisection-based algorithm that also exploits the unimodality of the reward function.

\item IMED-MB~\citep{saber2024bandits}: a multimodal bandit algorithm that exploits the existence of a small number of peaks in the reward function. Here, we set the number of assumed peaks to be ten.
\end{itemize}

Throughout our experiments, for \pretc, we set $M$, the number of random exploration rounds, to be 20.

\subsection{Evaluation on DeepMIMO simulation environment}

We first evaluate our algorithm and the baselines on DeepMIMO~\citep{alkhateeb2019deepmimo}, a generic ray-tracing-based simulated dataset for benchmarking beam alignment algorithms.
With DeepMIMO, we can vary the User Equipment (UE) location, allowing us to evaluate the algorithms across a wide range of locations.

We utilize a stationary scenario 4 from an urban area in Phoenix, which includes 3 base stations and 6,794 users. In total, we have 20,382 BS-UE pairs (bandit instances), including 4,952 pairs with a valid connection. 
For each BS-UE pair with a valid connection, we evaluate all algorithms' performance by simulating bandit learning and record their cumulative regret. 
We set the base station to have a ULA array with 16 antennas elements; we set each user to have one omnidirectional antenna.
The frequency of the mmWave in this experiment is 28GHz, thus, the wavelength is $\lambda = c/f = 0.011$ meters. 
The spacing of the antenna is $d = 0.5 \lambda = 0.005$. 
The size of the beam codebook is $K = 180$. We follow~\citep{samimi2014characterization} and set the standard deviation in the reward noise to be $\sigma=3.6$ (dBm).

We define the candidate set of angles of arrivals to be $\Theta = \cbr{2i, i \in{0, 1, \cdots, 180}}$ and the candidate set of channel gains to be $B = \cbr{e^{-10 + 0.5i}, i \in \cbr{0, 1, \cdots, 20}}$; our \pretc and \prgreedy algorithm takes $k = 1$ as input.

\subsubsection{Distribution of cumulative regret over different BS-UE pairs}
\Cref{fig:ave-regret-50,fig:ave-regret-200} show the distribution of 
N-Regret$_T$
over 4,952 environments for time horizons $T = 50$ and $T=200$.
For each environment, the N-Regret$_T$ is computed by averaging over 10 runs.

We find that \prgreedy consistently outperforms other algorithms. At time step 50, the expected mean regret of \prgreedy is 0.27, while \pretc has a mean regret of 0.59 (\Cref{fig:ave-regret-50}). 
Increasing the time horizon from 50 to 200 significantly improves the performance of both \prgreedy and \pretc (\Cref{fig:ave-regret-200}), although \prgreedy requires more time to solve the optimization problem. 
This is despite the fact that 2,851 BS-UE pairs have more than two paths, and 2,152 pairs have more than 3 paths. This demonstrates the robustness of our algorithm, even with $k=1$.

For the other unimodal bandit algorithms, we empirically observe that they do not perform as well as our physics-informed algorithms. 
LSE’s performance improves as the horizon $T$ increases from 50 to 200, approaching that of \pretc, but it still underperforms compared to \prgreedy.
BISECTION, designed for unimodal settings, struggles in both $T=50$ and $T=200$, as it is not well-suited for the multimodal nature of the problem.
Although IMED-MB is designed for multimodal settings, its regrets over all tasks exhibit a bimodal distribution in the histogram plots.
The experiment shows that IMED-MB struggles to select the best beam with only a very small number of observations.
We speculate that this is because they can identify the best beam in some environments, but in others, their overall performance is similar to a random policy.
UCB seems to have a clearly inferior performance due to its need in choosing every beam once to begin with; with the total number of beams $K = 180$ and a time horizon of $T = 50$, it performs almost identical to a random policy.

\begin{figure}[t]
    \centering
    \includegraphics[width=1.0\linewidth]{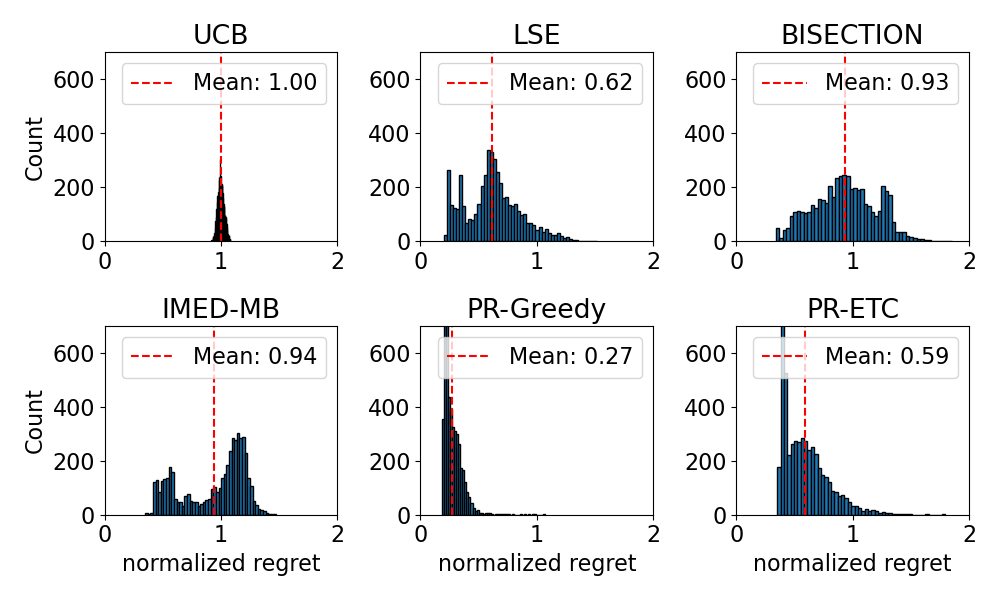}
    \caption{
    Distribution of normalized regret over all BS-UE pairs at the first $50$ steps.}
    \label{fig:ave-regret-50}
\end{figure}
\begin{figure}[t]
    \centering
    \includegraphics[width=1.0\linewidth]{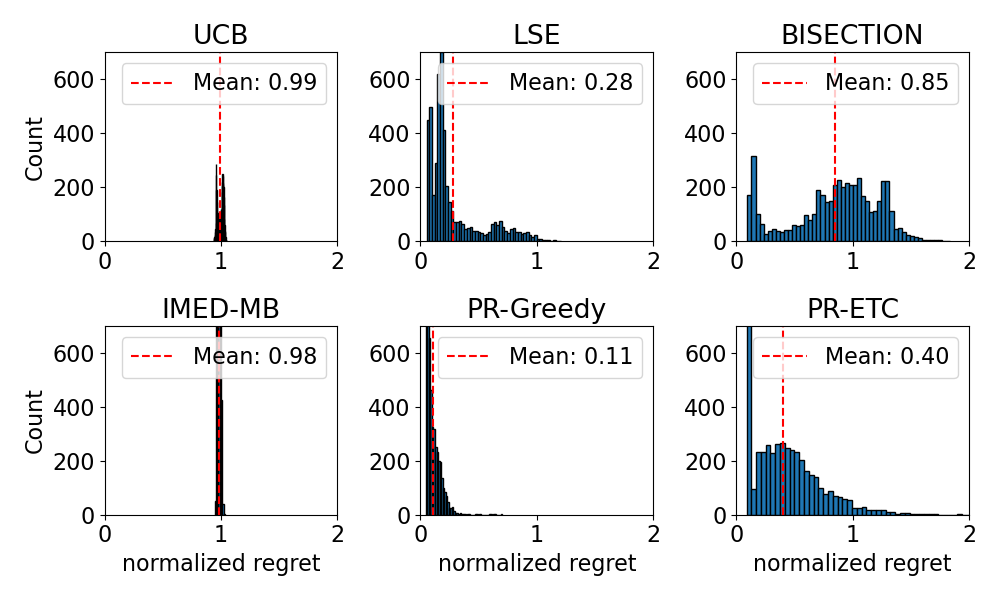}
    \caption{
    Distribution of normalized regret over all BS-UE pairs at the first $200$ steps.}
    \label{fig:ave-regret-200}
\end{figure}

\subsubsection{Case Study in one BS-UE location} \label{sec:rebuilt-reward-function}

To further understand the behavior of the algorithms and baselines, we select one bandit instance from all user locations (Base Station 1 with UE 45) as an example. From \Cref{fig:regret-vs-time}, we see that both \prgreedy and \pretc achieve lower average regret, which continues to decrease as time progresses and more data is collected, and \prgreedy performs better than \pretc.
BISECTION and LSE learn some patterns, but their performance remains inferior compared to our methods. 
UCB and IMED-MB perform the worst with no meaningful learning progress.

\begin{figure}[t]
    \centering 
    \includegraphics[width=0.9\linewidth]{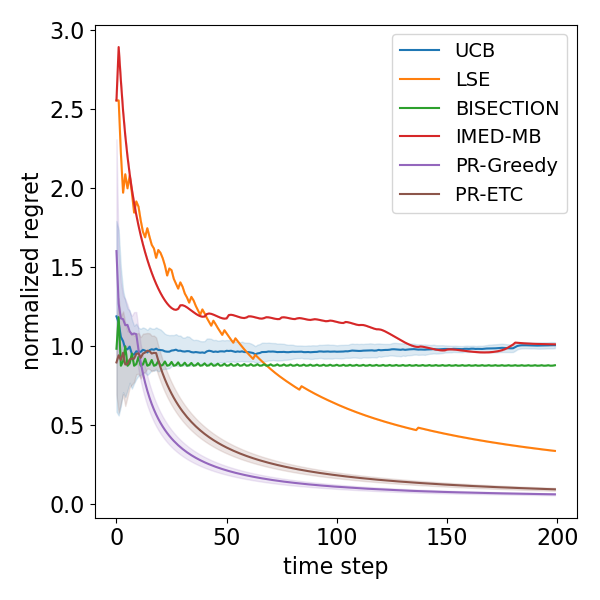}
    \caption{
    Normalized regret vs time step for different algorithms in DeepMIMO scenario 4 base station 1 UE 45. The curves are based on the average over 10 repetitions, with their shades indicating one standard error.
    Because of the finite-horizon setting with $T=200$ steps and $K=180$ arms, UCB and IMED-MB have a regret close to choosing beams uniformly at random. 
    }
    \label{fig:regret-vs-time}
\end{figure}

We also plot the ground truth reward function in \Cref{fig:rebuilt-reward} and the estimated reward function using the $(a_t, r_t)$ pairs collected by \prgreedy and \pretc.
From the plot, we can see that the beams chosen by \prgreedy are centered around the peak of the reward function rather than uniformly distributed on the action space.
\begin{figure}[t]
    \centering
    \includegraphics[width=0.9\linewidth]{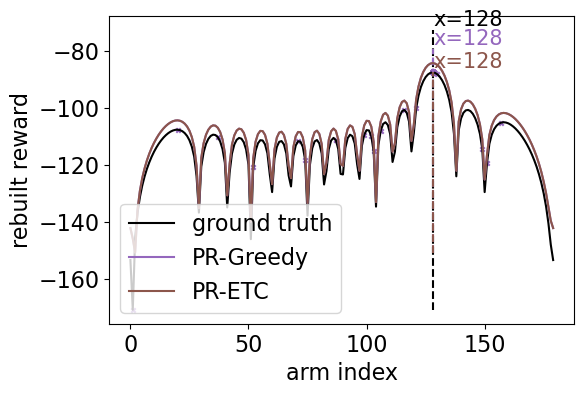}
    \caption{
    Reward function estimated by \prgreedy and \pretc at time step 200. Each cross represents the arm selected by \prgreedy and \pretc.
    For each color, the vertical lines represent the best beam that maximizes the respective reward function. 
    }
    \label{fig:rebuilt-reward}
\end{figure}

\subsubsection{Spatial distribution of normalized regret} \label{sec:regret-map-glm}

We also plot the spatial heatmap of the N-Regret$_T$ for Base Station 1, which contains 2,704 BS-UE pairs (\Cref{fig:bs-1-glm-greedy}).
We observe that the performance of \prgreedy aligns well with the histogram in \Cref{fig:ave-regret-200}, indicating generally strong performance.

\subsubsection{Time cost}
We compare the total computation time required by each algorithm for $T=200$, averaged over 10 repeated trials from all DeepMIMO scenarios.
Notably, \prgreedy takes 376.73 ms per step in one BS-UE pair, \pretc takes 7.91 ms, other baseline models complete within 0.5 ms per step, except IMED-MB which takes around 5 ms.
The majority of the time is spent by our algorithms to solve the MLE (\Cref{eqn:optimization-target}) by enumerating all $(\bm \theta, \bm \beta)$ pairs at every time step. 
This is still reasonable in practice, as the beam update interval has a median of 160-310 ms in real-world commercial deployments, according to \citep{feng2025vivisecting}.
The time cost of \prgreedy and \pretc can be improved with more efficient optimization methods or better hardware. Recent work also proposes to solve the beam alignnment problem using kernel bandits~\cite{deng2022interference}, which does not take advantage of the underlying channel's property.
\cite{alemdar2025remarkable} proposed employing reconfigurable intelligent surfaces instead of multiple wireless access points. Their algorithm uses a kernelized bandit model for beam selection, with a bandit-over-bandit approach over adaptive restarts to handle non-stationary scenarios. Our physics-informed online learning algorithm is related to bandit phase retrieval~\citep{kotlowski2019bandit,lattimore2021bandit} as the reward feedback is only related to the signal strength but not the phase shift. Our \prgreedy and \pretc algorithms use Maximum Likelihood Estimation (MLE) to estimate the parameters of the environment, similar to prior works in parametric bandits~\cite{filippi2010parametric}.
\begin{figure}[t]
    \centering
    \includegraphics[width=0.9\linewidth]{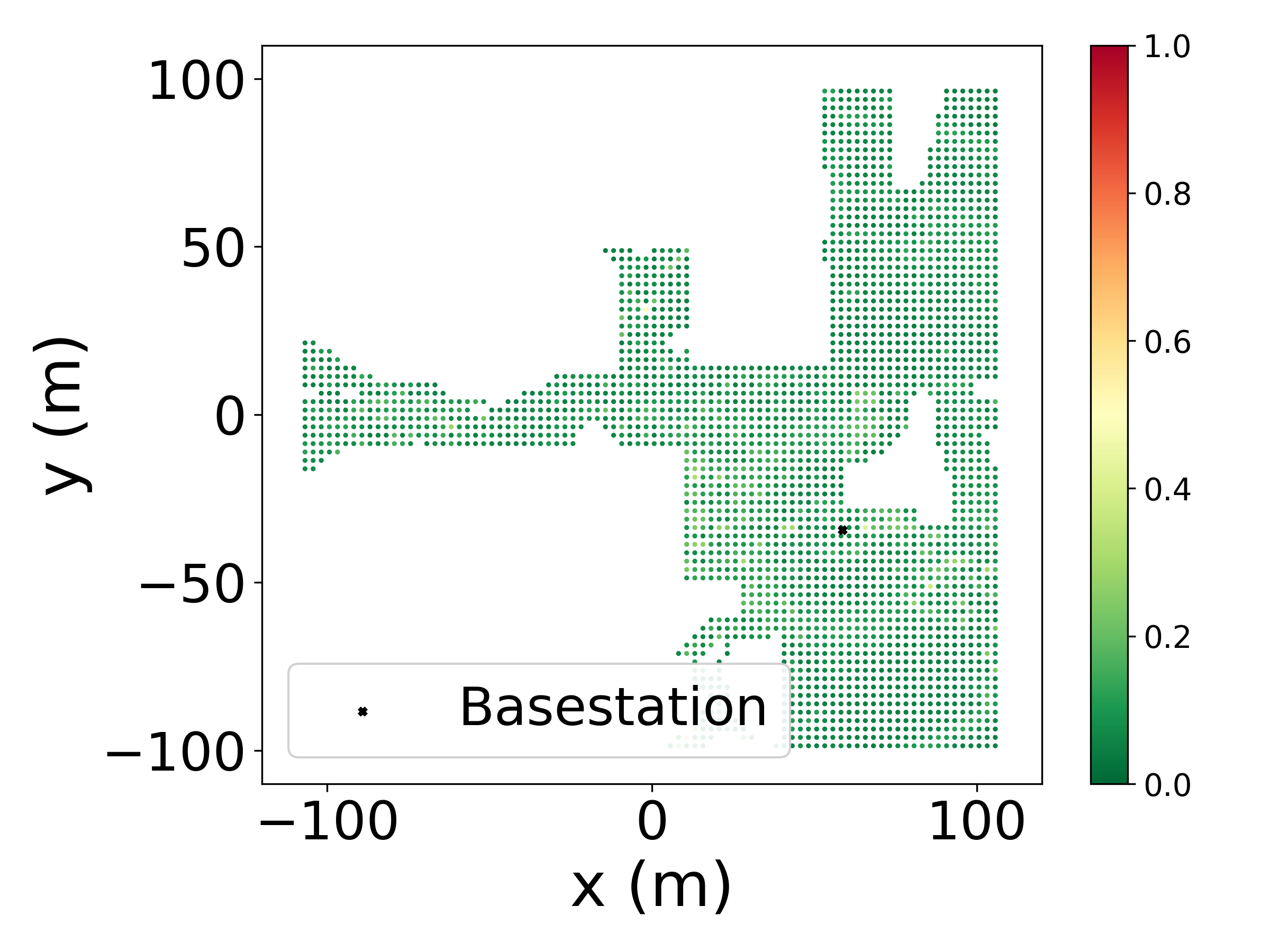}
    \caption{Spatial heat map of N-Regret$_T$ for \prgreedy across the coverage area of Base Station 1
    } \label{fig:bs-1-glm-greedy}
\end{figure}

\subsection{Experiments on DeepSense6G dataset}

We next evaluate our algorithms and baselines in the publicly-available DeepSense6G dataset~\citep{DeepSense} collected in Phoenix. 

\subsubsection{Stationary Reward Distribution Setting}

Out of the 44 scenarios provided by DeenSense6G, we focus on scenarios 17-22 and 24-29 that are outdoor scenarios in two-way city streets with fixed Rx and Tx locations.  
In these scenarios, the base station uses a 16-element ULA antenna, with a carrier frequency 60GHz.
The base station chooses one of $K=64$ beams from the codebook to communicate with the user. 
The codebook of beamforming vectors is not included in the metadata\footnote{The codebook is collected using a non-standard method by DeepSense6G, whose phase shifts on each codeword is unknown to us.}.
Nevertheless, the beam pattern of all beamforming vectors was provided in the metadata, which suffices for us to run our algorithms as we discussed in Section~\ref{sec:prelims}. For our channel estimation, the set of candidate AoA's is $\Theta = \cbr{2i, i \in [0, \cdots, 110]}$, and 
the set of candidate  $\beta$'s is $B = \cbr{ e^{-3 + 0.2i}, i \in \cbr{0, \cdots, 13} }$. 

We randomly sample 10 distinct length-200 intervals, each using a different random seed, from the union of all available bean RSS measurement sequences in each scenario to construct 10 independent trials.
For each trial, we run all algorithms and record their cumulative regret.
The raw RSS values from the dataset are directly used to simulate reward feedback.
Since the RSS values are recorded every 0.1 seconds, we set each decision step to correspond to 0.1 seconds as well.
Each simulation is conducted with a time horizon of $T = 200$.
For the \pretc and \prgreedy algorithms, we specify the number of assumed paths as $k = 1$ and $k = 2$, respectively.

We show the histogram of the normalized regret across these 12 scenarios in \Cref{fig:deepsense-ave-regret-50,fig:deepsense-ave-regret-200}. Similar to the trend in the \Cref{fig:ave-regret-50}, we also see that \prgreedy and \pretc has regret performance much better than all baselines.
As a case study verification, we also plot the per-step regret curves of all algorithms in scenario 17 in \Cref{fig:lc-scenario-24}. 
Overall, our experiments confirm the utility of physics-informed models in designing a better algorithm for beam selection. 

\begin{figure}[t]
    \centering
    \includegraphics[width=1.0\linewidth]{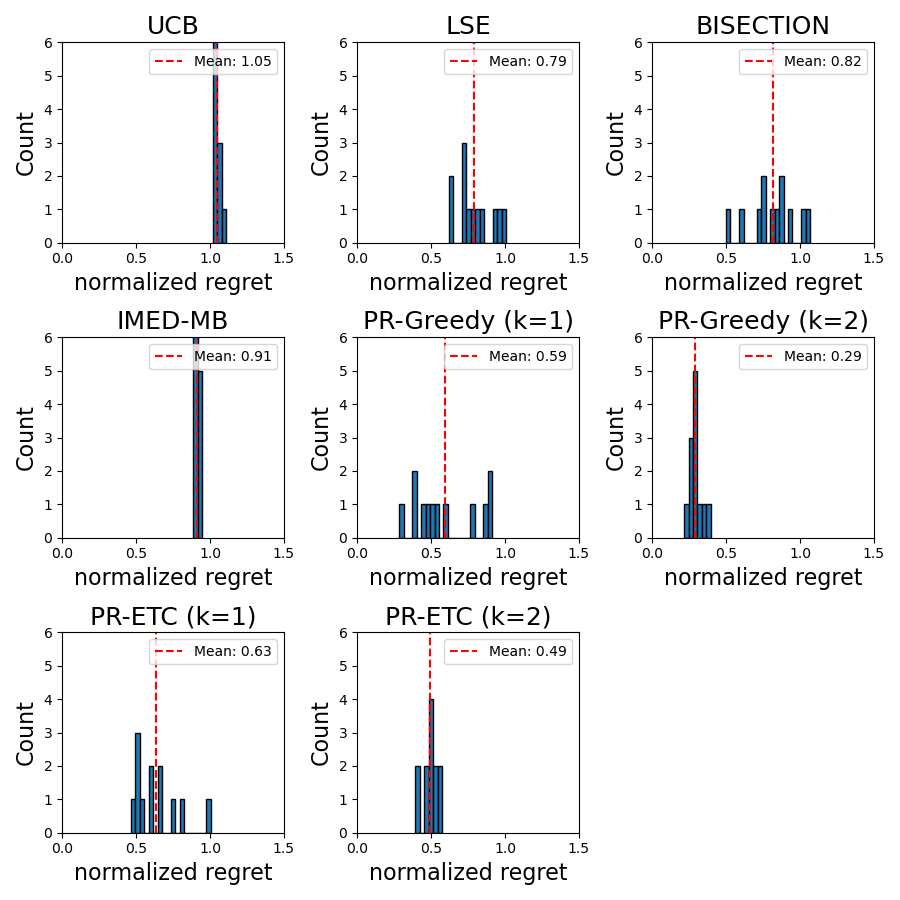}
    \caption{
    Distribution of N-Regret$_T$ over all DeepSense 6G static scenarios at the first $50$-th step}
    \label{fig:deepsense-ave-regret-50}
\end{figure}

\begin{figure}[t]
    \centering
    \includegraphics[width=1.0\linewidth]{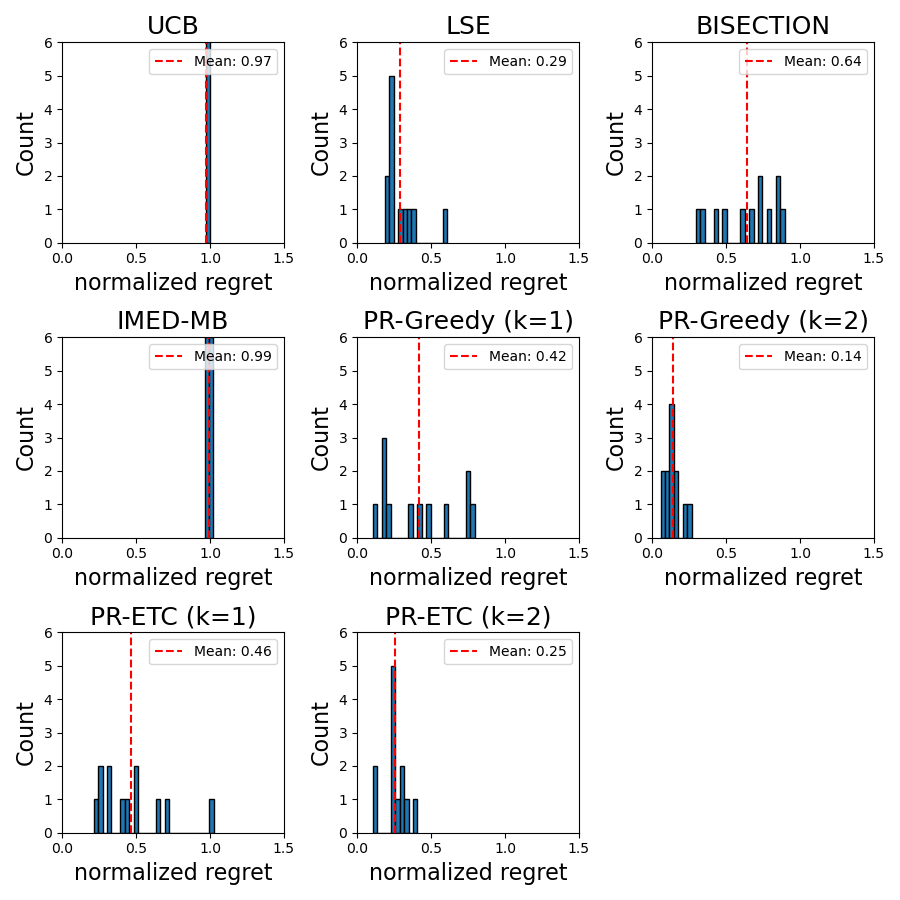}
    \caption{
    Distribution of N-Regret$_T$ over all DeepSense 6G static scenarios at the first $200$-th step.
    }
    \label{fig:deepsense-ave-regret-200}
\end{figure}

\begin{figure}[t]
    \centering 
    \includegraphics[width=0.9\linewidth]{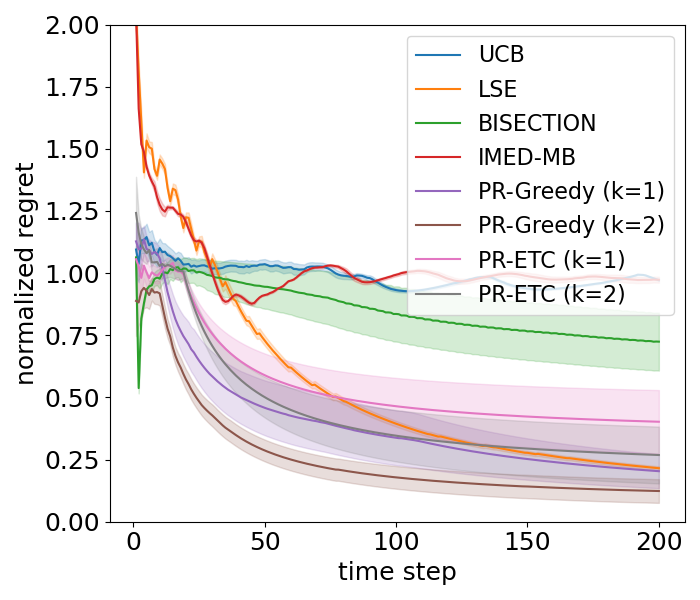}
    \caption{Average regret vs time step for different algorithms in Scenario 24 of DeepSence 6G, with $T = 200$ and $K = 64$ beams.
    All learning curves are based on the average over 10 repetitions with their shades indicating one standard error.}
    \label{fig:lc-scenario-24}
\end{figure}

\begin{figure}[t] 
    \centering
    \includegraphics[width=1.0\linewidth]{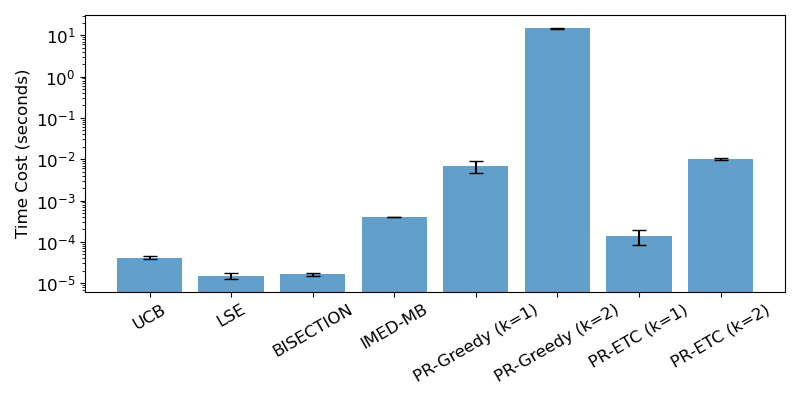}
    \caption{Per step time cost of each algorithm, averaged over all DeepSense 6G stationary scenarios by running $T = 200$ steps with 10 repeats.
    }
    \label{fig:time-cost-200}
\end{figure}
We also report the average time cost of all algorithms for $T=200$, across the 12 scenarios in \Cref{fig:time-cost-200}. As shown, the time cost of \pretc is significantly lower than that of \prgreedy.
Furthermore, increasing the number of paths $k$ from 1 to 2 leads to higher computation time due to the use of grid search, where the number of search grids increases exponentially with $k$.

\subsubsection{Changing Reward Distribution Setting}

We evaluate the effectiveness of Periodic-$\Acal$~(Algorithm~\ref{alg:restart}) in Scenario 9 of the DeepSense6G dataset. 
This is a vehicle-to-infrastructure case where the transmitter is located at the roadside and the receiver is in the moving vehicle.
The dataset records the signal strength for all 64 beams every 100 ms.
Each consecutive sequence contains about $40$ records, so each sequence is about $4$ seconds. 
We use this data to create a nonstationary bandit problem with a time horizon $T = 2,000$, with each time step representing an interval of 2 ms.
Since we only have RSS data for a coarser resolution of 100 ms, we linearly interpolate the reward function at every time step according to the reward function of its closest two 100ms time interval ticks available in the data.
We evaluate the restarted variants of \prgreedy and \prgreedy with $\tau = 50$; for baselines, we consider the restarted variants of UCB, LSE, BISECTION, IMED-MB with the same setting of $\tau$. 

We present the normalized dynamic regret 
\begin{align}
& \text{N-}\mathrm{DynamicRegret}_T \nonumber
    \coloneqq \\
    & \frac{\mathrm{DynamicRegret}_T}{\sum_{t=1}^T \max_{a \in [K]} R(\fbf_a, \bm \theta_t^*, \bm \beta_t^*) - \EE_{a \sim U([K])}R(\fbf_{a}, \bm \theta_t^*, \bm \beta_t^*) }
    \label{eqn:dynamic-normalized-regret}
\end{align}
as a function of time step in \Cref{fig:normalized-regret-mobile}; we see that Periodic-\prgreedy has the best performance and Periodic-\pretc is the second. Both beat all other periodic bandit algorithms by a significant margin.  

\begin{figure}[t]
    \centering
    \includegraphics[width=1.0\linewidth]{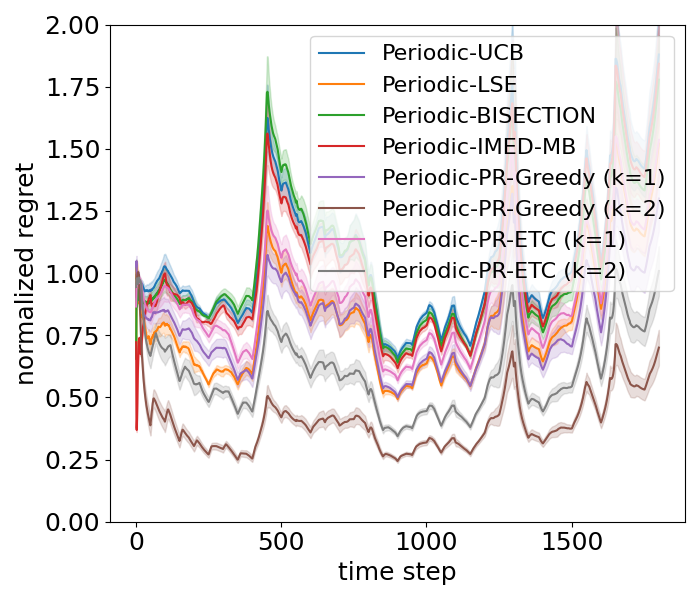}
    \caption{Learning curve averaged from 5 sequences in the dynamic scenario 9 with 1 standard error}
    \label{fig:normalized-regret-mobile}
\end{figure}

\section{Conclusion}

We develop physics-informed parametric bandit algorithms, \pretc and \prgreedy, to address the beam alignment problem in millimeter-wave communications. Unlike other works that rely on the unimodality or multimodality assumptions on the reward function, our approach is grounded in the fundamental model of far-field mmWave propagation with a small number of dominant paths. Our algorithms demonstrate robust performance on synthetic and real datasets under a variety of channel conditions, compared to existing model-free methods. 
We highlight several promising avenues for future work: (1) improve the time cost of our algorithms, possibly by utilizing compressed sensing~\citep{xie2015hekaton,wei2017facilitating}; (2) extend our approach to near-field communications~\citep{liu2023near}; (3) design more adaptive algorithms for the nonstationary reward distribution setting.

{\small
\bibliographystyle{IEEEtranN}
\bibliography{learning}
}

\newpage
\onecolumn
\appendix

\section{Appendix}

\subsection{Regret Analysis}

In this section, we prove \Cref{thm:reg-etc} for \pretc. To this end, we first bound the estimation error of the reward function $R$ on the dataset $S_M$ collected in the first $M$ steps in \Cref{lemma:bounding-estimation-error-by-data-norm}, then we show that the instantaneous regret of \pretc is bounded in \Cref{lemma:bounding-simple-regret}, and finally we show that the cumulative regret of \pretc is bounded in \Cref{thm:reg-etc}.

\medskip

\subsubsection{Auxiliary tools from MLE theory}
Let $\sse_M(\bm \theta, \bm \beta) \coloneq \sum_{t=1}^M \del{R(\fbf_t, \bm \theta, \bm \beta) - r_t}^2$, which is the sum of the square errors of the reward predictor $R(\cdot, \bm \theta, \bm \beta) \in \Rcal$ measured on the dataset collected by the first $M$ steps of \pretc.
Recall the definition of a dataset of $m$ action-reward pairs $S_m \coloneq \cbr{a_t, m_t}_{t=1}^m$ from Section IV.1, the data norm of a function $g$ based on the dataset $S_M$ is defined as

\[
    \opnorm{g}{S_M} \coloneq \sqrt{\sum_{t=1}^M g(\fbf_{a_t})^2}.
\]
where $g$ is a function that takes into steering vector $\fbf \in \CC^N$ as input.

Next, we bound the sum of squared errors of a reward function $R(\cdot, \bm \theta, \bm \beta)$ relative to the ground truth $R(\cdot, \bm \theta^*, \bm \beta^*)$ measured based on the dataset collected from the first $M$ step. Formally, we have \Cref{lemma:bounding-estimation-error-by-data-norm} to bound this estimation error.

\begin{lemma}[Concentration of Data Norm] \label{lemma:bounding-estimation-error-by-data-norm}
    Suppose for every time step $t$, $r_t = R(\fbf_{a_t}, \bm \theta^*, \bm \beta^*) + \eta_t$, where $\eta_t \sim N(0, \sigma^2)$'s are independent, where $\sigma$ is a known positive constant. There exists an event $E$, $P(E) \geq 1 - \delta$, and on $E$, the following two equations hold for all $(\bm \theta, \bm \beta) \in B^k \times \Theta^k$:
    \begin{align}
    \sse_{M} (\bm \theta, \bm \beta) - \sse_{M} (\bm \theta^*, \bm \beta^*)
    \geq&
    \frac{1}{2} \opnorm{R(\cdot, \bm \theta, \bm \beta) - R(\cdot, \bm \theta^*, \bm \beta^*)}{S_{M}}^2 - 4k\sigma^2\log\del{\frac{\abs{B} \abs{\Theta}}{\delta}}
        \label{eqn:upper-bound-of-error-on-data-norm}
    \\
    \sse_{M} (\bm \theta, \bm \beta) - \sse_{M} (\bm \theta^*, \bm \beta^*)
    \leq&
    \frac{3}{2} \opnorm{R(\cdot, \bm \theta, \bm \beta) - R(\cdot, \bm \theta^*, \bm \beta^*)}{S_{M}}^2 + 4k\sigma^2\log\del{\frac{\abs{B} \abs{\Theta}}{\delta}}
        \label{eqn:lower-bound-of-error-on-data-norm}
\end{align}
\end{lemma}

\begin{proof}[Proof of \Cref{lemma:bounding-estimation-error-by-data-norm}]
    Since the parameter space $B^k \times \Theta^k$ has a total size of $\abs{B}^k \times \abs{\Theta}^k$, we take a union bound on all parameter combinations from the parameter space and let $\delta' = \delta / \idel{\abs{B}^k \times \abs{\Theta}^k}$, and it suffices to show that for every $\del{\bm \theta, \bm \beta} \in B^k \times \Theta^k$, the following equations hold with probability at least $1 - \delta'$,
    \begin{align}
    \sse_{M} (\bm \theta, \bm \beta) - \sse_{M} (\bm \theta^*, \bm \beta^*)
    \geq&
    \frac{1}{2} \opnorm{R(\cdot, \bm \theta, \bm \beta) - R(\cdot, \bm \theta^*, \bm \beta^*)}{S_{M}}^2- 4\sigma^2\log\del{\frac{1}{\delta'}}
        \label{eqn:intermediate-upper-bound-of-error-on-data-norm}
    \\
    \sse_{M} (\bm \theta, \bm \beta) - \sse_{M} (\bm \theta^*, \bm \beta^*)
    \leq&
    \frac{3}{2} \opnorm{R(\cdot, \bm \theta, \bm \beta) - R(\cdot, \bm \theta^*, \bm \beta^*)}{S_{M}}^2 + 4\sigma^2\log\del{\frac{1}{\delta'}}
        \label{eqn:intermediate-lower-bound-of-error-on-data-norm}
    \end{align}

    We first prove the \Cref{eqn:intermediate-upper-bound-of-error-on-data-norm}. 
    Starting from the difference of $\sse$:
    \begin{align*}
        & \sse_{M} (\bm \theta, \bm \beta) - \sse_{M} (\bm \theta^*, \bm \beta^*)
        \\
        =&
        \sum_{t=1}^{M} \del{R(\fbf_{a_t}, \bm \theta, \bm \beta) - r_t}^2
        -
        \sum_{t=1}^{M} \del{R(\fbf_{a_t}, \bm \theta^*, \bm \beta^*) - r_t}^2
        \\
        =&
        \sum_{t=1}^{M} \del{R(\fbf_{a_t}, \bm \theta, \bm \beta) - R(\fbf_{a_t}, \bm \theta^*, \bm \beta^*)}^2
        -
        2 \sum_{t=1}^{M} \del{R(\fbf_{a_t}, \bm \theta, \bm \beta) - R(\fbf_{a_t}, \bm \theta^*, \bm \beta^*)}\eta_t
                    \tag{$r_t \coloneq R(\fbf_{a_t}, \bm \theta^*, \bm \beta^*) + \eta_t$ where $\eta_t \sim \Ncal(0, \sigma^2), \forall t \in [M]$ and algebra
            }
        \\
        =&
        \opnorm{R(\cdot, \bm \theta, \bm \beta) - R(\cdot, \bm \theta^*, \bm \beta^*)}{S_{M}}^2 - \Ecal_{M}
            \tag{Let $\Ecal_{M} \coloneq
        2 \sum_{t=1}^{M} \del{R(\fbf_{a_t}, \bm \theta, \bm \beta) - R(\fbf_{a_t}, \bm \theta^*, \bm \beta^*)}\eta_t$}
    \end{align*}
    Notice that the data norm $\opnorm{R(\cdot, \bm \theta, \bm \beta) - R(\cdot, \bm \theta^*, \bm \beta^*)}{S_{M}}^2$ is fixed for any given $R$ and $S_M$. The second term, $\Ecal_M$, is the sum of zero-mean random variables $\cbr{\eta_t}_{t=1}^{M}$ with coefficients.
    Conditioned on $\cbr{a_t}_{t=1}^T$, (1) $\Ecal_M$ has expectation zero; (2) $\Ecal_M$ is sub-Gaussian with variance proxy $4 \opnorm{R(\cdot, \bm \theta, \bm \beta) - R(\cdot, \bm \theta^*, \bm \beta^*)}{S_{M}}^2 \sigma^2$.
    (1) is true since $\eta_t$ is a zero-mean random variable.
    (2) can be verified by analyzing $\EE\sbr{e^{ \del{\Ecal_M - \EE[\Ecal_M]} \lambda }}$ with any $ \lambda \in \RR$,
    
    \begin{align*}
        \EE\sbr{e^{ \del{\Ecal_M - \EE[\Ecal_M]} \lambda }}
        =&
        \EE\sbr{e^{ \Ecal_M \lambda }}
        =
        \prod_{t=1}^M \EE \sbr{ \expto{ 2 \del{R(\fbf_{a_t}, \bm \theta, \bm \beta) - R(\fbf_{a_t}, \bm \theta^*, \bm \beta^*)} \eta_t \lambda } }
            \tag{$\eta_t$ are independent of each other.}
        \\
        \leq&
        \prod_{t=1}^M \expto{\frac{4\del{ R(\fbf_{a_t}, \bm \theta, \bm \beta) - R(\fbf_{a_t}, \bm \theta^*, \bm \beta^*) }^2 \sigma^2 \lambda^2}{2}}
            \tag{$\eta_t$ is $\sigma^2$-sub-Gaussian.}
        \\
        =&
        \expto{ \frac{4\opnorm{R(\cdot, \bm \theta, \bm \beta) - R(\cdot, \bm \theta^*, \bm \beta^*)}{S_{M}}^2 \sigma^2 \lambda^2}{2} }
    \end{align*}
    Thus, we have verified (2).
    
    Let $\lambda = \tfrac{1}{4 \sigma^2}$, the above inequality becomes the following, 
    \begin{align*}
        & \EE \sbr{\expto{ \frac{\Ecal_{M}}{4\sigma^2} }} 
        \leq
        \expto{ \frac{\opnorm{R(\cdot, \bm \theta, \bm \beta) - R(\cdot, \bm \theta^*, \bm \beta^*)}{S_{M}}^2}{8\sigma^2} }
        \\
        \Rightarrow&
        \EE \sbr{\expto{ \frac{\Ecal_{M}}{4\sigma^2} - \frac{\opnorm{R(\cdot, \bm \theta, \bm \beta) - R(\cdot, \bm \theta^*, \bm \beta^*)}{S_{M}}^2}{8\sigma^2} } } 
        \leq 1
        \\
        \Rightarrow&
        \PP\del{\expto{ \frac{\Ecal_{M}}{4\sigma^2} - \frac{\opnorm{R(\cdot, \bm \theta, \bm \beta) - R(\cdot, \bm \theta^*, \bm \beta^*)}{S_{M}}^2}{8\sigma^2} } 
            \geq \frac{1}{\delta'}}
        \leq \delta'
            \tag{apply Markov inequality by taking $\expto{\ldots}$ as a random variable.}
        \\
        \Rightarrow&
        \PP\del{ \frac{\Ecal_{M}}{4\sigma^2} - \frac{\opnorm{R(\cdot, \bm \theta, \bm \beta) - R(\cdot, \bm \theta^*, \bm \beta^*)}{S_{M}}^2}{8\sigma^2} 
            \leq \log\del{\frac{1}{\delta'}} }
        \geq 1 - \delta'
        \\
        \Rightarrow&
        \PP\del{ \Ecal_{M} \leq \frac{1}{2}\opnorm{R(\cdot, \bm \theta, \bm \beta) - R(\cdot, \bm \theta^*, \bm \beta^*)}{S_{M}}^2 + 4 \sigma^2 \log\del{\frac{1}{\delta'}} }
        \geq 1 - \delta'
    \end{align*}

    In other words, we have that with probability at least $1 - \delta'$,
    \begin{align*}
        \sse_{M} (\bm \theta, \bm \beta) - \sse_{M} (\bm \theta^*, \bm \beta^*)
        =&
        \opnorm{R(\cdot, \bm \theta, \bm \beta) - R(\cdot, \bm \theta^*, \bm \beta^*)}{S_{M}}^2
        -\Ecal_{M}
        \\
        \geq&
        \opnorm{R(\cdot, \bm \theta, \bm \beta) - R(\cdot, \bm \theta^*, \bm \beta^*)}{S_{M}}^2
        - \frac{1}{2}\opnorm{R(\cdot, \bm \theta, \bm \beta) - R(\cdot, \bm \theta^*, \bm \beta^*)}{S_{M}}^2 
        - 4\sigma^2\log\del{\frac{1}{\delta'}}
        \\
        =&
        \frac{1}{2}\opnorm{R(\cdot, \bm \theta, \bm \beta) - R(\cdot, \bm \theta^*, \bm \beta^*)}{S_{M}}^2 - 4\sigma^2\log\del{\frac{1}{\delta'}}
            \tag{Let $\lambda = \fr14$.}
    \end{align*}

        \Cref{eqn:intermediate-lower-bound-of-error-on-data-norm} is derived by using a very similar way by letting $\Ecal'_{M} = -\Ecal_{M}$, and the downstream analysis will go through by replacing $\Ecal_{M}$ by $\Ecal'_{M}$ and adjusting the inequality accordingly.

\end{proof}

\Cref{lemma:bounding-estimation-error-by-data-norm} helps us control the in-sample prediction error in terms of squared data norm. Next, we would like to control the out-of-sample error, which is the expected error over the uniform distribution of the $K$ beams. To this end, we use Chernoff's bound (\Cref{lemma:chernoff-inequality}) that relates in-sample error to out-of-sample error. Specifically, we will use an inequality (\Cref{eqn:control-out-of-sample-mean}) implied by the standard Chernoff bound.

\begin{lemma}[Chernoff's inequality] \label{lemma:chernoff-inequality}
    For i.i.d. random variables $X_1, \dots, X_M$ with a bounded support set $X \in [0, A]$. Denote its mean $\mu = \EE[X]$ and empirical mean $\bar{X}_M \coloneq \tfrac{1}{M} \sum_{i=1}^M X_i$. Then for any $\varepsilon > 0$, we have
    \begin{align}
        \PP\del{ \bar{X}_M < \mu - \varepsilon }
        \leq&
        \expto{ - \frac{M\varepsilon^2}{2 A (\mu + \varepsilon)} }
            \label{eqn:tighter-chernoff-lower-tail-bound}
    \end{align}
    and consequently,
    \begin{align}
        \PP\del{ \mu < 2\bar{X}_M + \frac{8A}{M} \log\del{\frac{1}{\delta}} } \geq 1 - \delta
            \label{eqn:control-out-of-sample-mean}
    \end{align}
\end{lemma}
    
\begin{proof}
    We start from the standard Chernoff's bound on the random variable $\bar{X}_M$:
    \begin{align}
        \PP\del{ \bar{X}_M < \mu - \varepsilon} 
        \leq& 
        \expto{-M \KL{\frac{\mu-\varepsilon}{A}}{\frac{\mu}{A}}}
            \label{eqn:chernoff-lower-tail-bound}
    \end{align}
    where $\KL{\mu_1}{\mu_2}$ represents the KL divergence between two Bernoulli distributions with means $\mu_1$ and $\mu_2$, respectively.
    \Cref{eqn:chernoff-lower-tail-bound} is a standard Chernoff's bound for bounded random variables and we skip the proof.
    Then we are going to show \Cref{eqn:tighter-chernoff-lower-tail-bound} based on \Cref{eqn:chernoff-lower-tail-bound}.
    From \Cref{eqn:chernoff-lower-tail-bound}, it suffices to show that the following inequality holds
    \[
        \KL{\frac{\mu + \varepsilon}{A}}{\frac{\mu}{A}} \geq \frac{\varepsilon^2}{2A(\mu+\varepsilon)}.
    \]
    Since we know that $\KL{a}{b} \geq \frac{\del{a - b}^2}{2 \max\cbr{a, b}}$, we have
    \[
        \KL{\frac{\mu + \varepsilon}{A}}{\frac{\mu}{A}}
        \geq
        \frac{\del{\frac{\mu+\varepsilon}{A} - \frac{\mu}{A}}^2}{2 \max\cbr{\frac{\mu+\varepsilon}{A}, \frac{\mu}{A}}}
        =
        \frac{\frac{\varepsilon^2}{A^2}}{2 \frac{\mu+\varepsilon}{A}}
        =
        \frac{\varepsilon^2}{2A(\mu+\varepsilon)}
    \]
    Plug in the above inequality into \Cref{eqn:chernoff-lower-tail-bound}, we prove that \Cref{eqn:tighter-chernoff-lower-tail-bound} holds.

    Next, we are going to show \Cref{eqn:control-out-of-sample-mean} based on \Cref{eqn:tighter-chernoff-lower-tail-bound}.
    To finish the proof, we need to find an appropriate $\varepsilon$ such that $\expto{ - \frac{M\varepsilon^2}{2 A (\mu + \varepsilon)} } \leq \delta$. 
    Based on the basic inequality $a + b \leq 2 (a \vee b)$, where $a, b \geq 0$ and $\vee$ means maximum, \Cref{eqn:tighter-chernoff-lower-tail-bound} implies the following:
    \begin{align*}
        \PP\del{ \bar{X}_M < \mu - \varepsilon }
        \leq&
        \expto{- \frac{M \varepsilon^2}{2A \cdot 2(\mu \vee \varepsilon)} }
        =
        \expto{ - \frac{M\varepsilon^2}{4 \mu A} } \vee \expto{ -\frac{M \varepsilon}{ 4 A} }
    \end{align*}
    We let the RHS be at most $\delta$:
    \begin{align*}
        \varepsilon \geq& \sqrt{ \frac{4 \mu A \log\frac1\delta}{M} } \vee \frac{4 A \log\frac1\delta}{M}
        \\
        \Leftarrow \varepsilon
        \geq&
        \del{ \fr\mu2 + \frac{2 A \log\frac1\delta}{M} } \vee \frac{4 A \log\frac1\delta}{M}
            \tag{AM-GM inequality}
        \\
        \Leftarrow \varepsilon
        \geq&
        \fr\mu2 + \frac{4 A \log\frac1\delta}{M}
            \tag{Since $\mu \geq 0$}
    \end{align*}
    Therefore, by letting $\varepsilon = \fr\mu2 + \frac{4 A \log\frac1\delta}{M}$, $\expto{ - \frac{M\varepsilon^2}{2 A (\mu + \varepsilon)} } \leq \delta$.
    Finally, plugging this choice of $\varepsilon$ in \Cref{eqn:tighter-chernoff-lower-tail-bound} and we have that with probability at most $\delta$,
    \begin{align*}
        \bar{X}_M <& \mu - \varepsilon \\
        \Leftrightarrow
        \bar{X}_M <& \mu - \del{ \fr\mu2 + \frac{4 A \log\frac1\delta}{M} }
            \tag{Based on the definition of $\varepsilon$}
        \\
        \Leftrightarrow
        \mu >& 2\bar{X}_M + \frac{8 A \log\frac1\delta}{M}
    \end{align*}
    \Cref{eqn:control-out-of-sample-mean} follows.
\end{proof}

\medskip

\subsubsection{Bounding the regret for \pretc} We are ready to conclude the proof of \Cref{thm:reg-etc}. 

\begin{proof}[Proof of \Cref{thm:reg-etc}]

Denote the instantaneous regret at time step $t$ as 
\[
\regret_t := \max_{a \in [K]} R(\fbf_a, \bm \theta^*, \bm \beta^*) - R(\fbf_{a_t}, \bm \theta^*, \bm \beta^*).
\]

For the first $M$ steps, since the algorithm purely explores, we bound the regret of the first $M$ steps by $M$ times the maximum reward difference $2R_{\max}$. Therefore,
\begin{align*}
    \Regret_T =& \sum_{t=1}^{M} \regret_t + \sum_{t=M+1}^{T} \regret_t
    \\
    \leq&
    O\del{ M R_{\max} + T \sqrt{\frac{k \sigma^2 R_{\max}^2}{M}\log\del{\frac{B \abs{\Theta}}{\delta}}} }
    \\
    \leq&
    O\del{ R_{\max} T^{2/3} \del{k\sigma^2 \log\del{\frac{B\abs{\Theta}}{\delta}}}^{1/3} }
\end{align*}
where in the first inequality, we also apply \Cref{lemma:bounding-simple-regret} to bound each term in the second summation.
In the second inequality, we let $M = T^{2/3} \del{k\sigma^2 \log\del{\frac{B\abs{\Theta}}{\delta}}}^{1/3}$ to balance the two terms.
\end{proof}

\begin{lemma}[Instantaneous Regret of the Committed Beam] \label{lemma:bounding-simple-regret}
    Suppose Assumptions~\ref{assum:quadratic-lower-bound} and \ref{assum:lip} hold, with probability $1-\delta$, the instantaneous regret of \pretc in the steps of the exploitation phase is bounded by:
    \begin{align}
        \regret_{t}
        \leq \upbound{\sqrt{\frac{k \sigma^2 R_{\max}^2}{M} \log\del{\frac{\abs{B}\abs{\Theta}}{\delta}}}}, \forall M + 1 \leq t \leq T
    \end{align}
\end{lemma}

\begin{proof}
    Since \pretc does not update the estimation of parameters and the committed beam does not change in the exploitation phase, it suffices to prove that $\regret_{M+1} \leq \iupbound{ \sqrt{ \tfrac{k \sigma^2 R_{\max}^2}{M}\log\idel{\frac{\abs{B} \abs{\Theta}}{\delta}} } }$.
    For simplicity, we abbreviate $a_{M+1}$ as $\hat{a}$, $\hat{\bm \theta}_{M+1}$ as $\hat{\bm \theta}$ and $\hat{\bm \beta}_{M+1}$ as $\hat{\bm \beta}$.
    We also denote $a^* \coloneq \argmax_{a\in[K]} R(\fbf_a, \bm \theta^*, \bm \beta^*)$.

    First, we decompose the instantaneous regret as follows,
\begin{align*}
    & \regret_{M+1}
    \\
    =& 
        R(\fbf_{a^*}, \bm \theta^*, \bm \beta^*) - R(\fbf_{\hat{a}}, \bm \theta^*, \bm \beta^*) \\
    =& 
        \del{ R(\fbf_{a^*}, \bm \theta^*, \bm \beta^*) - R(\fbf_{a^*}, \hat{\bm \theta}, \hat{\bm \beta})
        }
        + 
        \del{ R(\fbf_{a^*}, \hat{\bm \theta}, \hat{\bm \beta}) - R(\fbf_{\hat{a}}, \hat{\bm \theta}, \hat{\bm \beta})
        }
        + 
        \del{ 
        R(\fbf_{\hat{a}}, \hat{\bm \theta}, \hat{\bm \beta}) - R(\fbf_{\hat{a}}, \bm \theta^*, \bm \beta^*)
        }
    \\
    \leq&
        \del{ R(\fbf_{a^*}, \bm \theta^*, \bm \beta^*) - R(\fbf_{a^*}, \hat{\bm \theta}, \hat{\bm \beta})
        }
        + 
        \del{ R(\fbf_{\hat{a}}, \hat{\bm \theta}, \hat{\bm \beta}) - R(\fbf_{\hat{a}}, \bm \theta^*, \bm \beta^*)
        }
\end{align*}
The last inequality is due to the fact that $R(\fbf_{\hat{a}}, \bm \hat{\theta}, \bm \hat{\beta}) \geq R(\fbf_{a}, \hat{\bm \theta}, \hat{\bm \beta}), \forall a \in \Acal$ and we upper bound the second difference $R(\fbf_{a^*}, \hat{\bm \theta}, \hat{\bm \beta}) - R(\fbf_{\hat{a}}, \hat{\bm \theta}, \hat{\bm \beta})$ by $0$.
Then we need to control the first and the last differences, which are the estimation error of the expected reward function on the optimal beam $a^*$ and the committed beam $\hat{a}$.
Let $\hat{R}$ denote a shorthand of $R(\cdot, \hat{\bm \theta}, \hat{\bm \beta})$.
Next, we define some favorable events in which the estimation error is well-controlled and show that these events hold with high probability.

\begin{align*}
    E_1 \coloneq& \cbr{ 
        \opnorm{R(\cdot, \hat{\bm \theta}, \hat{\bm \beta}) - R(\cdot, \bm \theta^*, \bm \beta^*)}{S_{M}}^2 
        \leq
        8k\sigma^2\log\del{\frac{ 2\abs{B}\abs{\Theta} }{\delta}}
     }
    \\
    E_2 \coloneq& \cbr{ \EE_{a \sim \text{Unif}(\Acal)} \sbr{ \del{R(\fbf_{a}, \hat{\bm \theta}, \hat{\bm \beta}) - R(\fbf_{a}, \bm \theta^*, \bm \beta^*)}^2 } 
        \leq 
        \frac{2}{M}\opnorm{R(\cdot, \hat{\bm \theta}, \hat{\bm \beta}) - R(\cdot, \bm \theta^*, \bm \beta^*)}{S_{M}}^2 
        + \frac{32 R_{\max}^2}{M} \log\del{\frac{2}{\delta}} }
\end{align*}

\paragraph{$E_1$}
According to \Cref{lemma:bounding-estimation-error-by-data-norm}, $E_1$ holds with probability at least $1 - \delta/2$.
When $E_1$ holds, we have
\begin{align}
    \frac{1}{2} \opnorm{R(\cdot, \hat{\bm \theta}, \hat{\bm \beta}) - R(\cdot, \bm \theta^*, \bm \beta^*)}{S_{M}}^2 - 4k\sigma^2\log\del{\frac{ 2\abs{B}\abs{\Theta} }{\delta}}
    \leq&
    \sse_{M}(\hat{\bm \theta}, \hat{\bm \beta}) - \sse_{M}(\bm \theta^*, \bm \beta^*)
        \nonumber
    \\
    \Rightarrow
    \frac{1}{2} \opnorm{R(\cdot, \hat{\bm \theta}, \hat{\bm \beta}) - R(\cdot, \bm \theta^*, \bm \beta^*)}{S_{M}}^2 - 4k\sigma^2\log\del{\frac{ 2\abs{B}\abs{\Theta} }{\delta}}
    \leq& 0
        \tag{$(\hat{\bm  \theta}, \hat{\bm  \beta})$ is the MSE estimator.}
        \nonumber
    \\
    \Rightarrow
    \opnorm{R(\cdot, \hat{\bm \theta}, \hat{\bm \beta}) - R(\cdot, \bm \theta^*, \bm \beta^*)}{S_{M}}^2
    \leq& 8k\sigma^2\log\del{\frac{ 2\abs{B}\abs{\Theta} }{\delta}}
        \label{eqn:intermediate-event-1}
\end{align}

\paragraph{$E_2$}
We will apply \Cref{lemma:chernoff-inequality} to show $E_2$ holds with probability at least $1 - \delta/2$. To this end, we let $\idel{R(\fbf_{a_i}, \hat{\bm \theta}, \hat{\bm \beta}) - R(\fbf_{a_i}, \bm \theta^*, \bm \beta^*)}^2$ be $X_i$ for every $i \in \cbr{0, 1, \dots, M}$, then the expectation $\mu$ and the bound of variable $A$ satisfy the following:
\begin{align*}
    \mu \coloneq& \EE_{a \sim \text{Unif}(\Acal)} \sbr{\del{R(\fbf_{a}, \hat{\bm \theta}, \hat{\bm \beta}) - R(\fbf_{a}, \bm \theta^*, \bm \beta^*)}^2}
    \\
    A 
    \coloneq& 4 R_{\max}^2 \geq \max_{a \in [K]} \del{ R(\fbf_a, \hat{\bm \theta}, \hat{\bm \beta}) - R(\fbf_a, \bm \theta^*, \bm \beta^*) }^2,
\end{align*}
where $R_{\max}$ is the maximum absolute value of $R$.
Then, according to \Cref{eqn:control-out-of-sample-mean} in \Cref{lemma:chernoff-inequality}, we have that with probability at least $1 - \delta/2$, $E_2$ holds.

For the rest of the proof, we condition on $E_1 \cap E_2$ happening, which occurs with probability $1 - \delta$ by union bound, we have
\begin{align}
    \mu \leq&
    \frac{2}{M}\opnorm{R(\cdot, \hat{\bm \theta}, \hat{\bm \beta}) - R(\cdot, \bm \theta^*, \bm \beta^*)}{S_{M}}^2 
        + \frac{32 R_{\max}^2}{M} \log\del{\frac{2}{\delta}}
            \tag{Event $E_2$}
    \\
    \leq&
    \frac{16k\sigma^2}{M} \log\del{\frac{ 2\abs{B}\abs{\Theta} }{\delta}} + \frac{32 R_{\max}^2}{M} \log\del{\frac{2}{\delta}}
            \tag{According to \Cref{eqn:intermediate-event-1}}
    \\
    =& \upbound{\frac{ k \sigma^2 R_{\max}^2}{M} \log\del{\frac{\abs{B}\abs{\Theta}}{\delta}}}
     \label{eqn:bound-expected-squared-error-gap}
\end{align}

Next, we bound the estimation error on beam $a^*$:
\begin{align*}
    & R(\fbf_{a^*}, \bm \theta^*, \bm \beta^*) - R(\fbf_{a^*}, \hat{\bm \theta}, \hat{\bm \beta})
    \\
    \leq&
    C_3 \abs{\bm \theta^* - \hat{\bm \theta}} + C_4 \abs{\log \bm \beta^* -\log \hat{\bm \beta}}
        \tag{Applying \Cref{assum:lip}}
    \\
    \leq&
    \frac{2\max\cbr{C_3, C_4}}{\min\cbr{\sqrt{C_1}, \sqrt{C_2}}} \cdot \frac{ \sqrt{C_1}\abs{\bm \theta^* - \hat{\bm \theta}} + \sqrt{C_2}\abs{\log \bm \beta^* - \log \hat{\bm \beta}}}{2}
    \\
    \leq&
    \frac{2\max\cbr{C_3, C_4}}{\min\cbr{\sqrt{C_1}, \sqrt{C_2}}} \cdot
    \sqrt{\frac{ C_1\abs{\bm \theta^* - \hat{\bm \theta}}^2 + C_2\abs{\log \bm \beta^* - \log \hat{\bm \beta}}^2 }{2}}
        \tag{AM-QM}
    \\
    \leq&
    \frac{2\max\cbr{C_3, C_4}}{\min\cbr{\sqrt{C_1}, \sqrt{C_2}}} \cdot \sqrt{\frac{\EE_{a \sim \text{Unif}(\Acal)} \del{R(\fbf_a, \bm \theta^*, \bm \beta^*) - R(\fbf_a, \hat{\bm \theta}, \hat{\bm \beta})}^2}{2}}
        \tag{By \Cref{assum:quadratic-lower-bound}}
\end{align*}

In the above derivation, we do not care about the constant term (such as $C_3$ and $C_4$, etc.). Then according to the \Cref{eqn:bound-expected-squared-error-gap}, we have
\[
    R(\fbf_{a^*}, \bm \theta^*, \bm \beta^*) - R(\fbf_{a^*}, \hat{\bm \theta}, \hat{\bm \beta})
    \leq
    \upbound{\sqrt{\frac{k \sigma^2 R_{\max}^2}{M} \log\del{\frac{\abs{B}\abs{\Theta}}{\delta}}}}
\]

Using a similar derivation, we also bound the estimation error on the committed beam $\hat{a}$ by
\begin{align*}
    R(\fbf_{\hat{a}}, \bm \theta^*, \bm \beta^*) - R(\fbf_{\hat{a}}, \hat{\bm \theta}, \hat{\bm \beta})
    \leq
    \upbound{\sqrt{\frac{k \sigma^2 R_{\max}^2}{M} \log\del{\frac{\abs{B}\abs{\Theta}}{\delta}}}}
\end{align*}

Finally, we conclude that
\begin{align*}
    & \regret_{M+1}
    \\
    \leq&
        \del{ R(\fbf_{a^*}, \bm \theta^*, \bm \beta^*) - R(\fbf_{a^*}, \hat{\bm \theta}, \hat{\bm \beta})
        }
        + 
        \del{ R(\fbf_{\hat{a}}, \hat{\bm \theta}, \hat{\bm \beta}) - R(\fbf_{\hat{a}}, \bm \theta^*, \bm \beta^*)
        }
    \\
    =&
    \upbound{\sqrt{\frac{k \sigma^2 R_{\max}^2}{M} \log\del{\frac{\abs{B}\abs{\Theta}}{\delta}}}}
\end{align*}

\end{proof}

\subsection{Proof of \Cref{thm:reg-pr-greedy}} \label{app:proof-reg-pr-greedy}

We introduce the following event $E_3$ to control the estimation error of the MSE estimator at time $T$:
\[
    E_3 \coloneq \cbr{
        \forall t \leq T,
        \opnorm{R(\cdot, \hat{\bm \theta}_t, \hat{\bm \beta}_t) - R(\cdot, \bm \theta^*, \bm \beta^*)}{S_{t-1}}^2 <  k \sigma^2 \log\del{\frac{T\abs{B} \abs{\Theta}}{\delta}}
    }.
\]

Event $E_3$ holds with probability greater than $1 - \delta$, which is a direct consequence of the concentration property of the MLE estimator. This is a consequence of applying \Cref{lemma:bounding-estimation-error-by-data-norm} with $\bm\theta$ and $\bm \beta$ being the MLE estimator at time $t$, $\hat{\bm \theta}_t, \hat{\bm \beta}_t$ for every $t \leq T$, and then applying a union bound over all time steps $t \leq T$.
Recall that we define $N_{t, a}$ to be the number of times that beam $a$ is chosen up to time $t$ (inclusively).

For the rest of the proof, we condition on the event $E_3$ happening.
We prove the theorem by contradiction: if at any time step $t \leq T$ a suboptimal beam $a$ is chosen and $N_{t-1, a} > \frac{k \sigma^2 \log(T\abs{B}\abs{\Theta})}{\gamma^2}$, then $\opnorm{R(\cdot, \hat{\bm \theta}_t, \hat{\bm \beta}_t) - R(\cdot, \bm \theta^*, \bm \beta^*)}{S_{t-1}}^2$ becomes larger than $k \sigma^2 \log(T\abs{B}\abs{\Theta})$, which contradicts with that $E_3$ happens:
    \begin{align*}
        &\quad \opnorm{R(\cdot, \hat{\bm \theta}_t, \hat{\bm \beta}_t) - R(\cdot, \bm \theta^*, \bm \beta^*)}{S_{t-1}}^2
        =
        \sum_{a \in \Acal}
        N_{t-1, a} \cdot \del{R(\fbf_{a}, \hat{\bm \theta}_{t}, \hat{\bm \beta}_{t}) 
        - R(\fbf_{a}, \bm \theta^*, \bm \beta^*)}^2
        \\
        &\geq
        N_{t-1, a_t} \cdot \del{R(\fbf_{a_t}, \hat{\bm \theta}_{t}, \hat{\bm \beta}_{t}) 
        - R(\fbf_{a_t}, \bm \theta^*, \bm \beta^*)}^2
        \geq
        N_{t-1, a_t} \cdot \gamma^2
        >
        k \sigma^2 \log(\abs{B}\abs{\Theta})
    \end{align*}
    The first inequality is because we only keep the terms when beam $a_t$ is chosen.
    The second inequality is due to \Cref{assum:self-identifiability}, by letting the suboptimal beam $a = a_t$, $\bm \theta = \hat{\bm \theta}_t$, and $\bm \beta = \hat{\bm \beta}_t$.
    The last inequality is due to that $N_{t-1, a_t} > \frac{k \sigma^2 \log(T\abs{B}\abs{\Theta})}{\gamma^2}$.
    This suggests that for any time step $t$ when a suboptimal beam $a_t$ is chosen more than $\frac{k \sigma^2 \log(T\abs{B}\abs{\Theta})}{\gamma^2}$ times, the estimation error of the MSE estimator at time $t$ becomes larger than $k \sigma^2 \log(T\abs{B}\abs{\Theta})$, which contradicts with the fact that $E_3$ holds.

    This contradiction implies that for any time step $t \leq T$, for any suboptimal beam $a$, $N_{t, a} \leq \frac{k \sigma^2 \log(T\abs{B}\abs{\Theta})}{\gamma^2}$, which further implies that for any suboptimal beam $a$, $N_{T, a} \leq \frac{k \sigma^2 \log(T\abs{B}\abs{\Theta})}{\gamma^2}$.

\newpage
\subsection{Regret heatmap for DeepMIMO} \label{sec:regret-maps}
We include the spatial heat maps of N-Regret$_T$ of other algorithms over the DeepMIMO dataset in \Cref{fig:cumulative-regret-map-matrix}. From the figures, it is clear that \prgreedy achieves the best overall performance, followed by \pretc and LSE.
\begin{figure}[H]
    \centering

    \begin{minipage}{0.12\textwidth}~\end{minipage}
    \begin{minipage}{0.28\textwidth}
        \centering
        \textbf{Base Station 1}
    \end{minipage}
    \begin{minipage}{0.28\textwidth}
        \centering
        \textbf{Base Station 2}
    \end{minipage}
    \begin{minipage}{0.28\textwidth}
        \centering
        \textbf{Base Station 3}
    \end{minipage}

    \vspace{0.5em}

    \begin{minipage}{0.12\textwidth}
        \raggedright
        \textbf{PR-ETC}
    \end{minipage}
    \begin{minipage}{0.28\textwidth}
        \includegraphics[width=\linewidth]{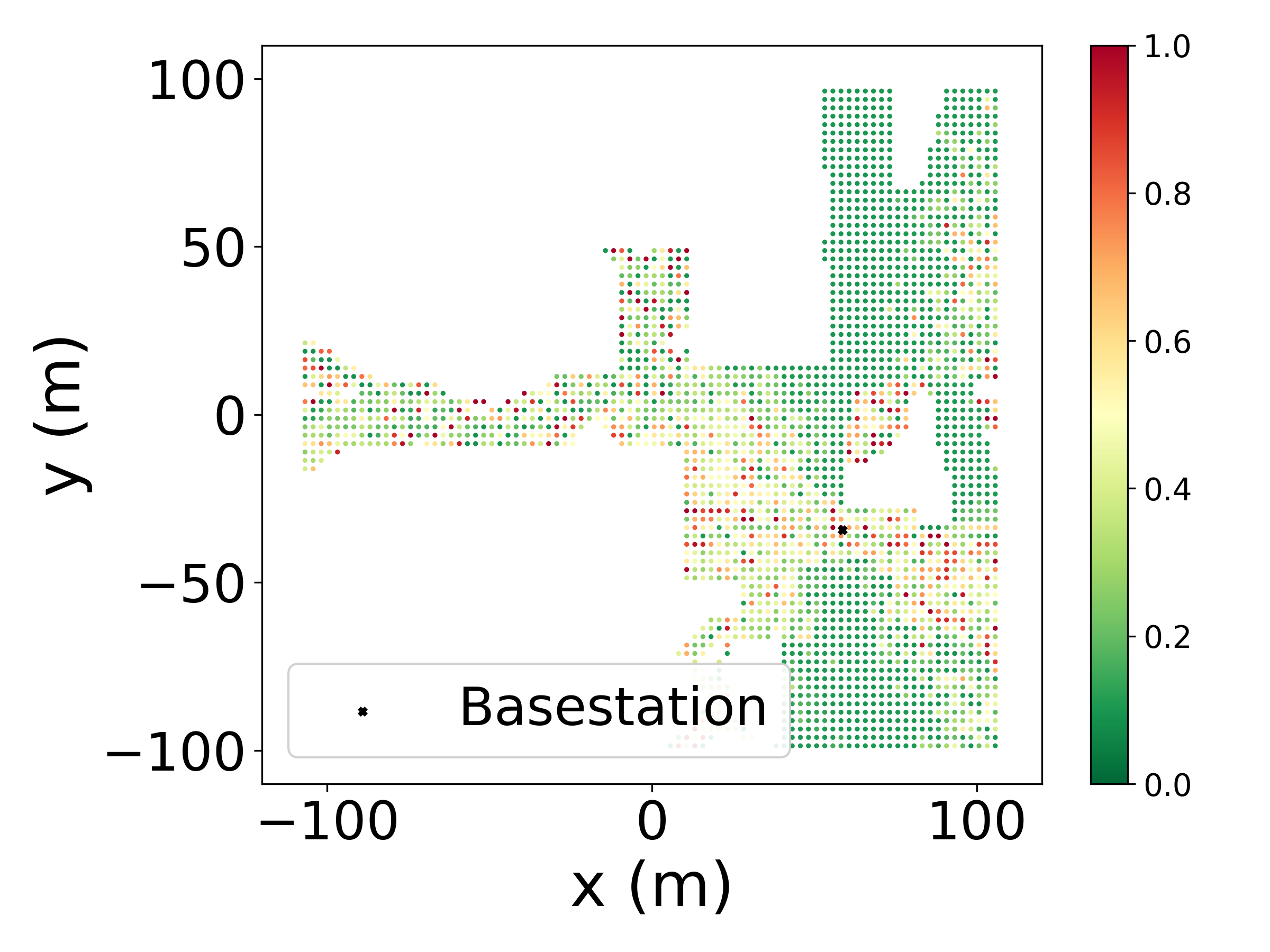}
    \end{minipage}
    \begin{minipage}{0.28\textwidth}
        \includegraphics[width=\linewidth]{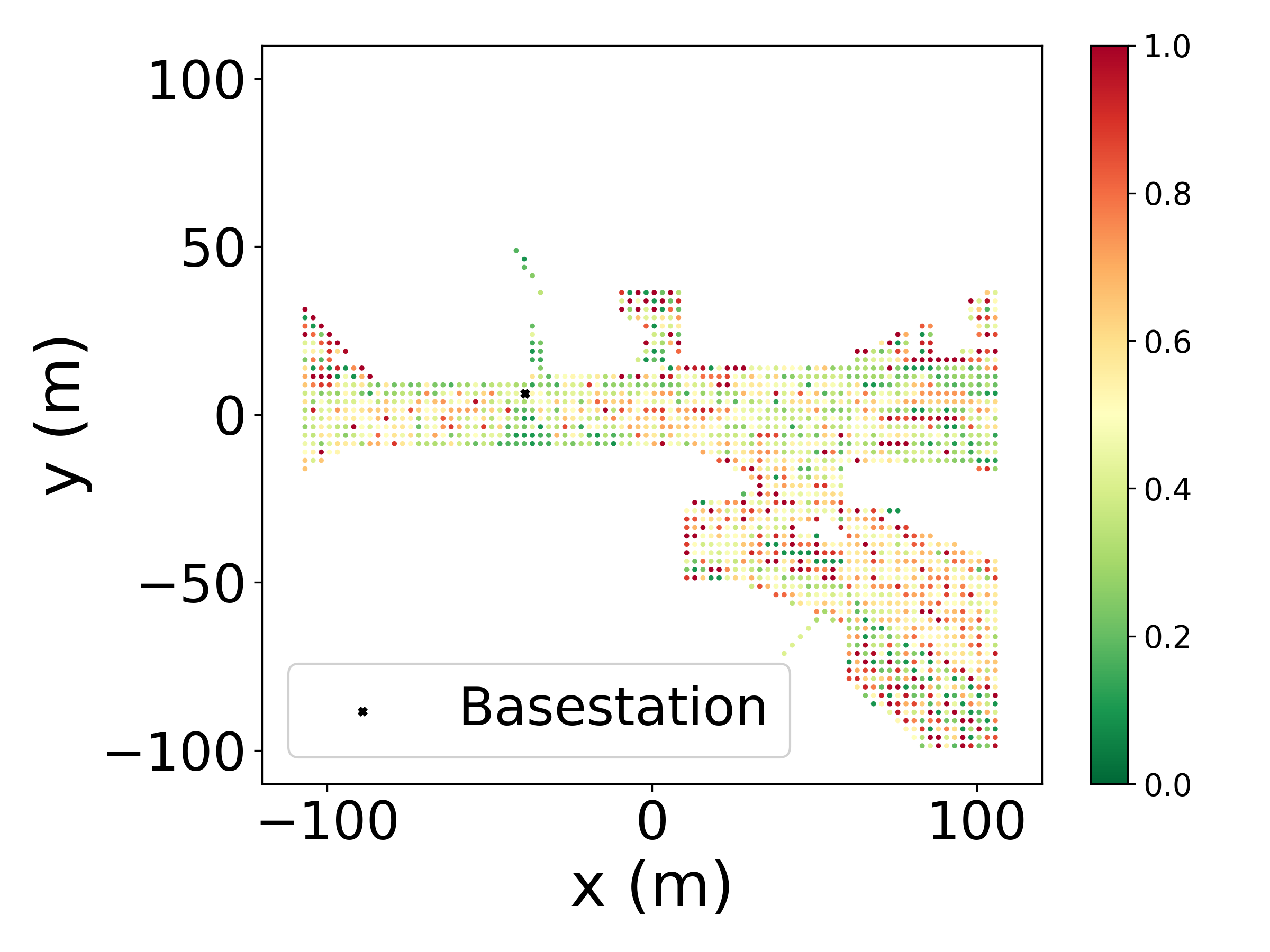}
    \end{minipage}
    \begin{minipage}{0.28\textwidth}
        \includegraphics[width=\linewidth]{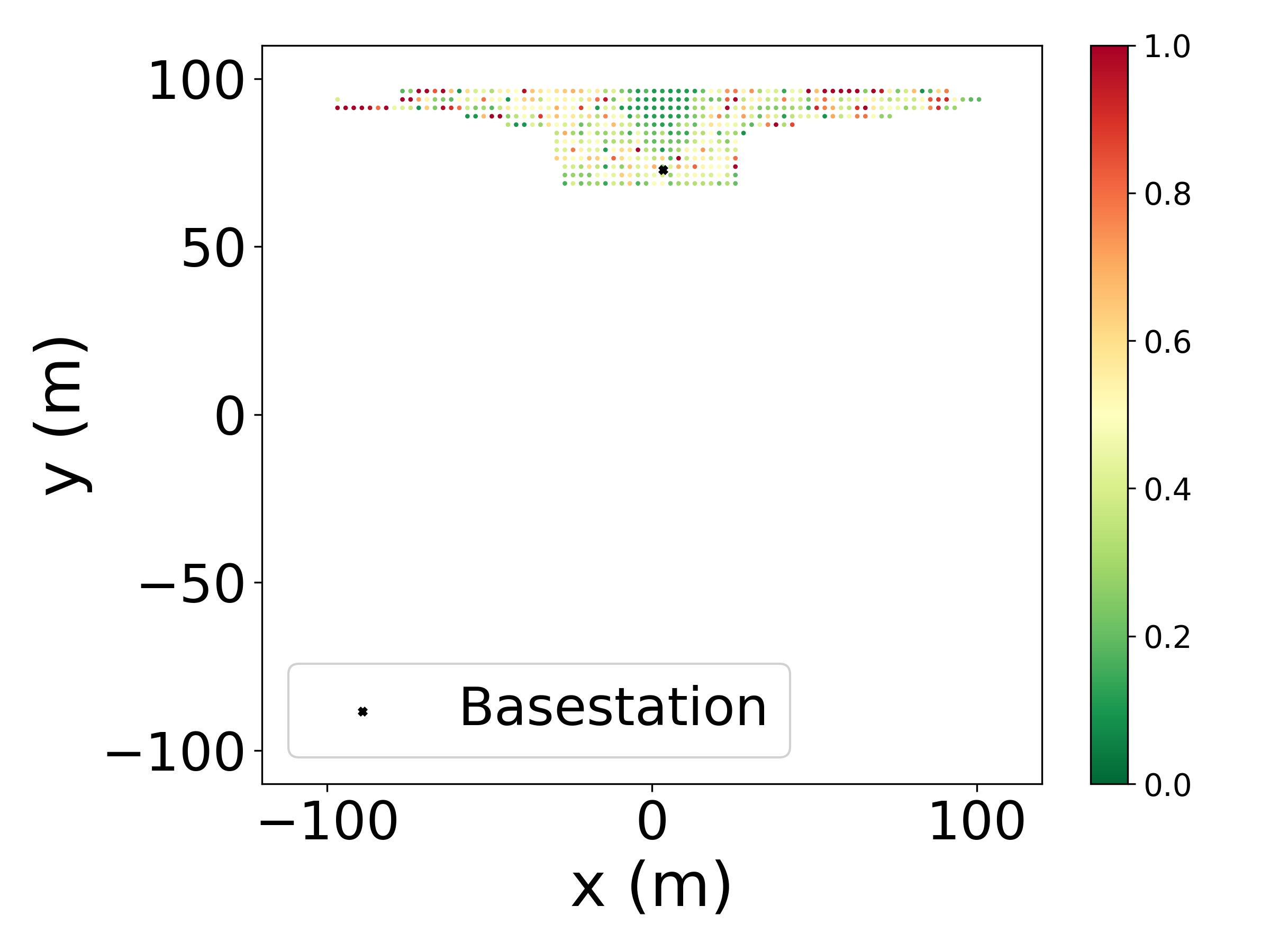}
    \end{minipage}

    \begin{minipage}{0.12\textwidth}
        \raggedright
        \textbf{LSE}
    \end{minipage}
    \begin{minipage}{0.28\textwidth}
        \includegraphics[width=\linewidth]{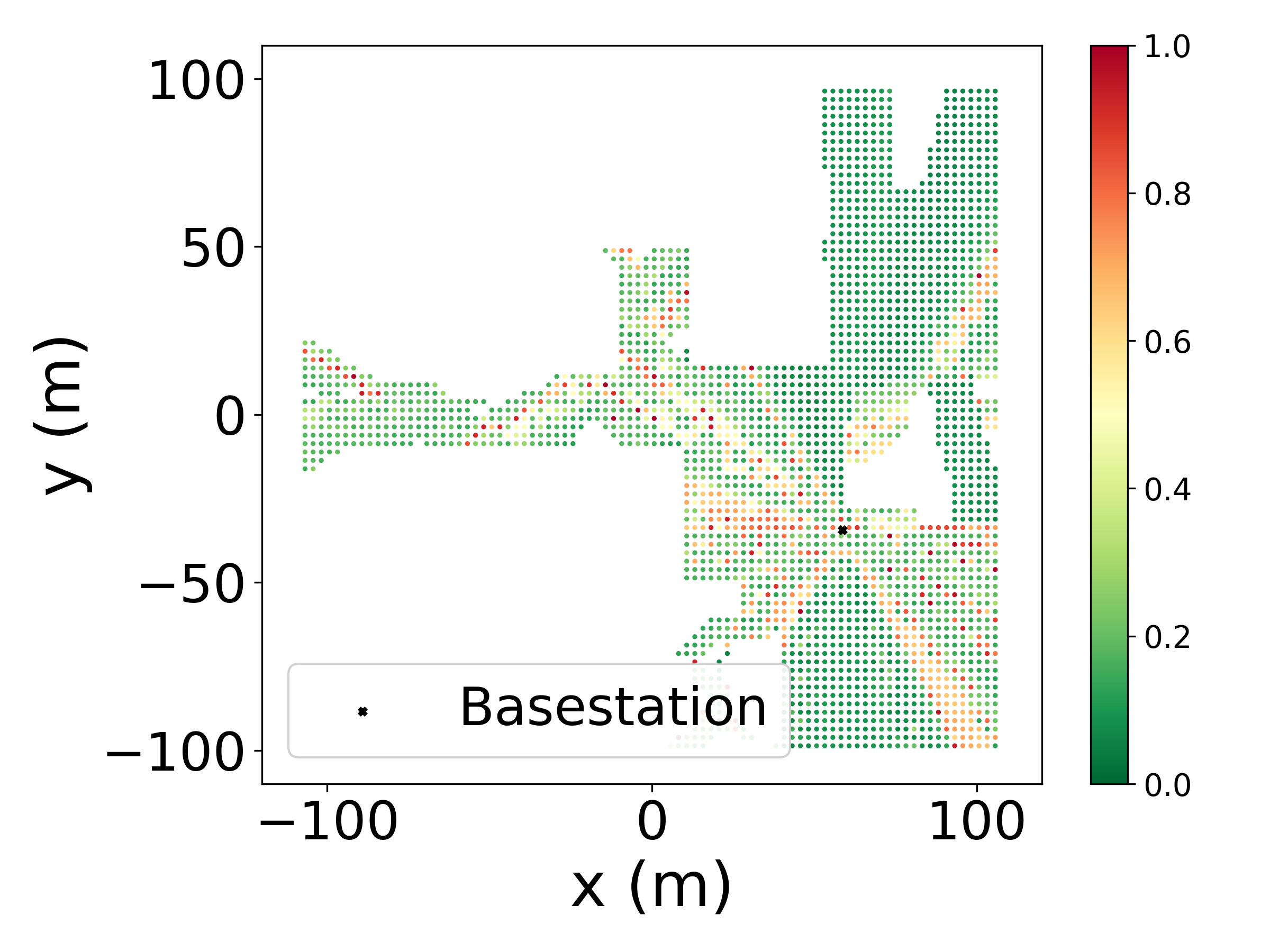}
    \end{minipage}
    \begin{minipage}{0.28\textwidth}
        \includegraphics[width=\linewidth]{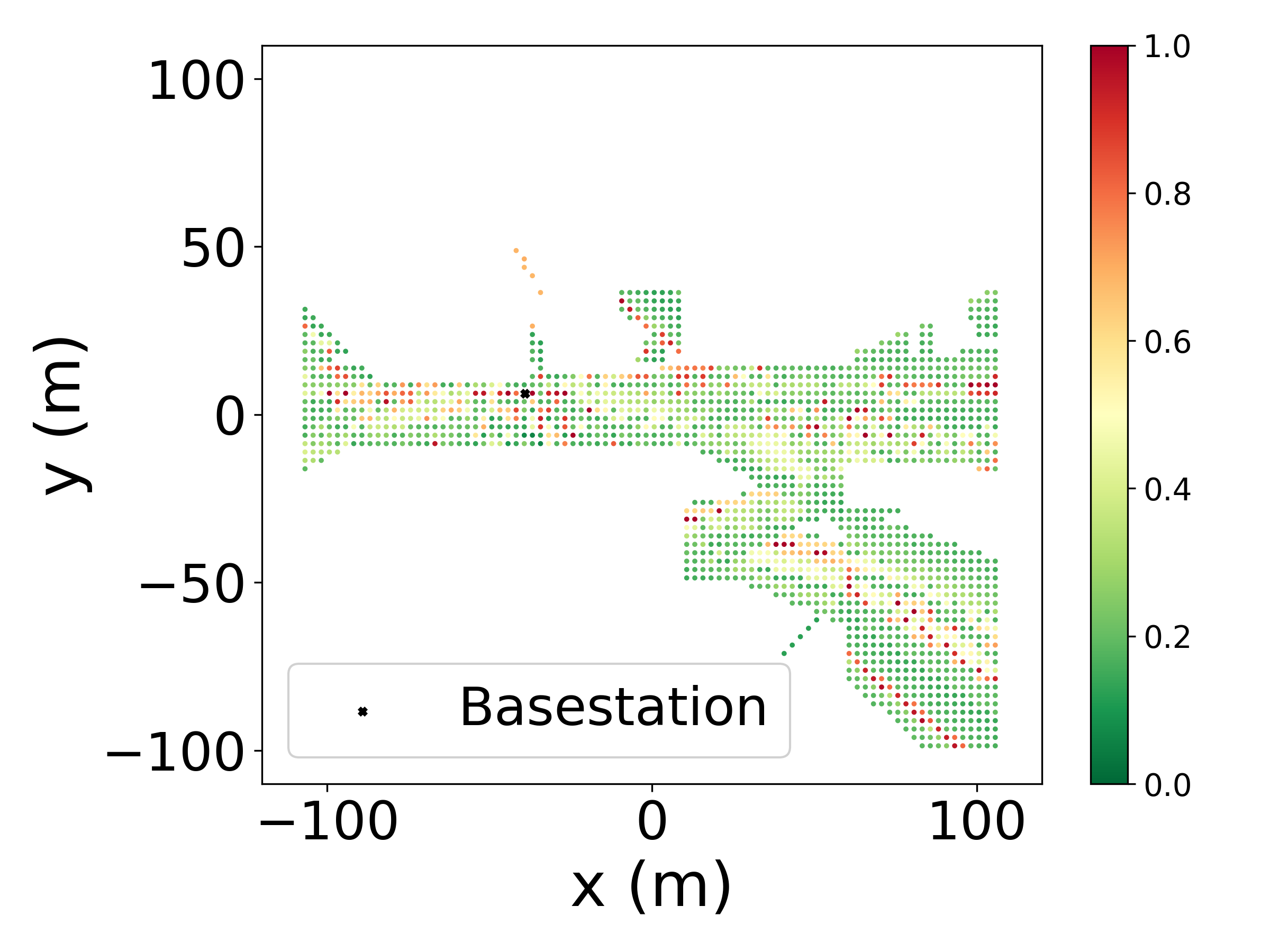}
    \end{minipage}
    \begin{minipage}{0.28\textwidth}
        \includegraphics[width=\linewidth]{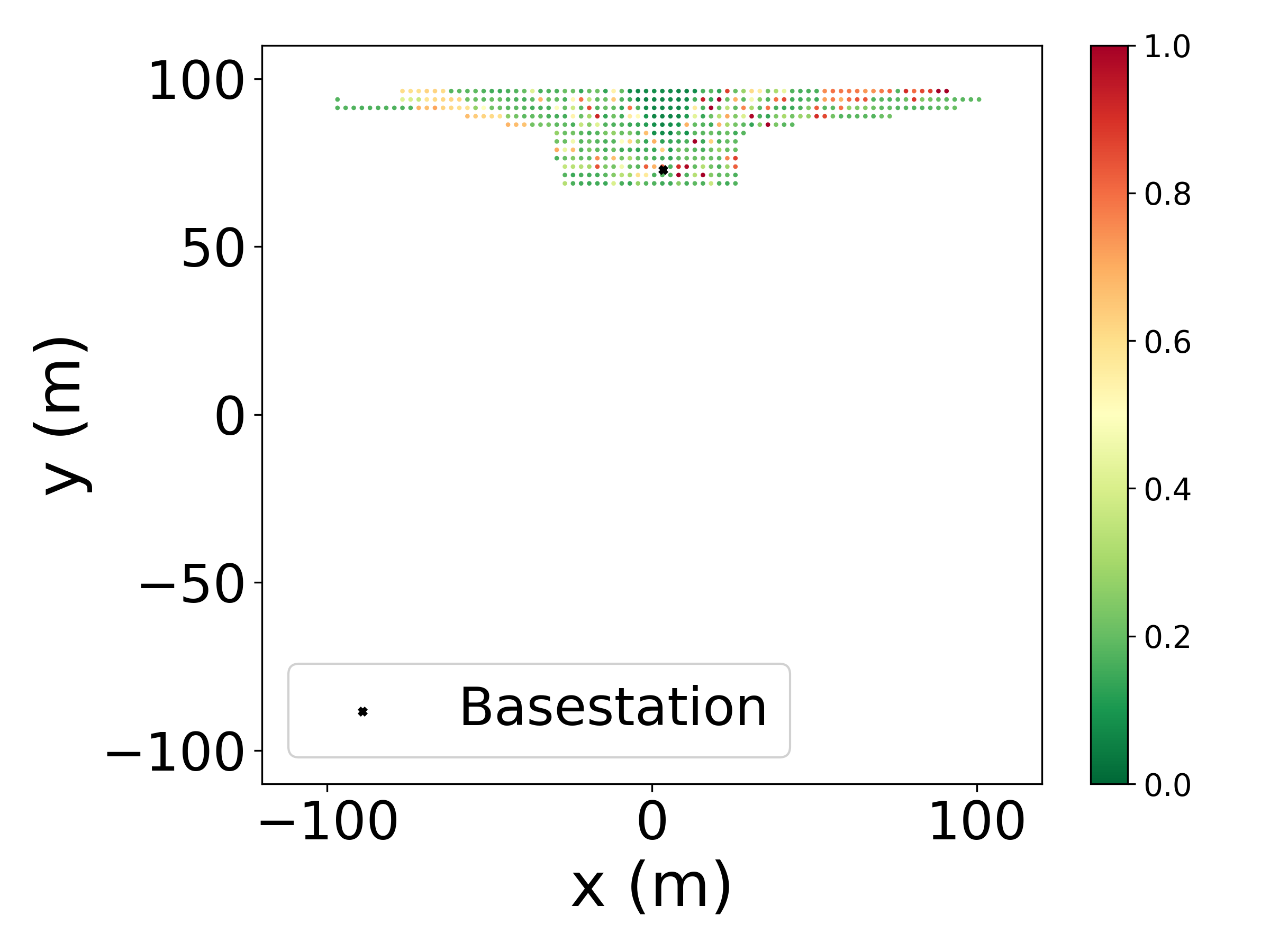}
    \end{minipage}

    \begin{minipage}{0.12\textwidth}
        \raggedright
        \textbf{BISECTION}
    \end{minipage}
    \begin{minipage}{0.28\textwidth}
        \includegraphics[width=\linewidth]{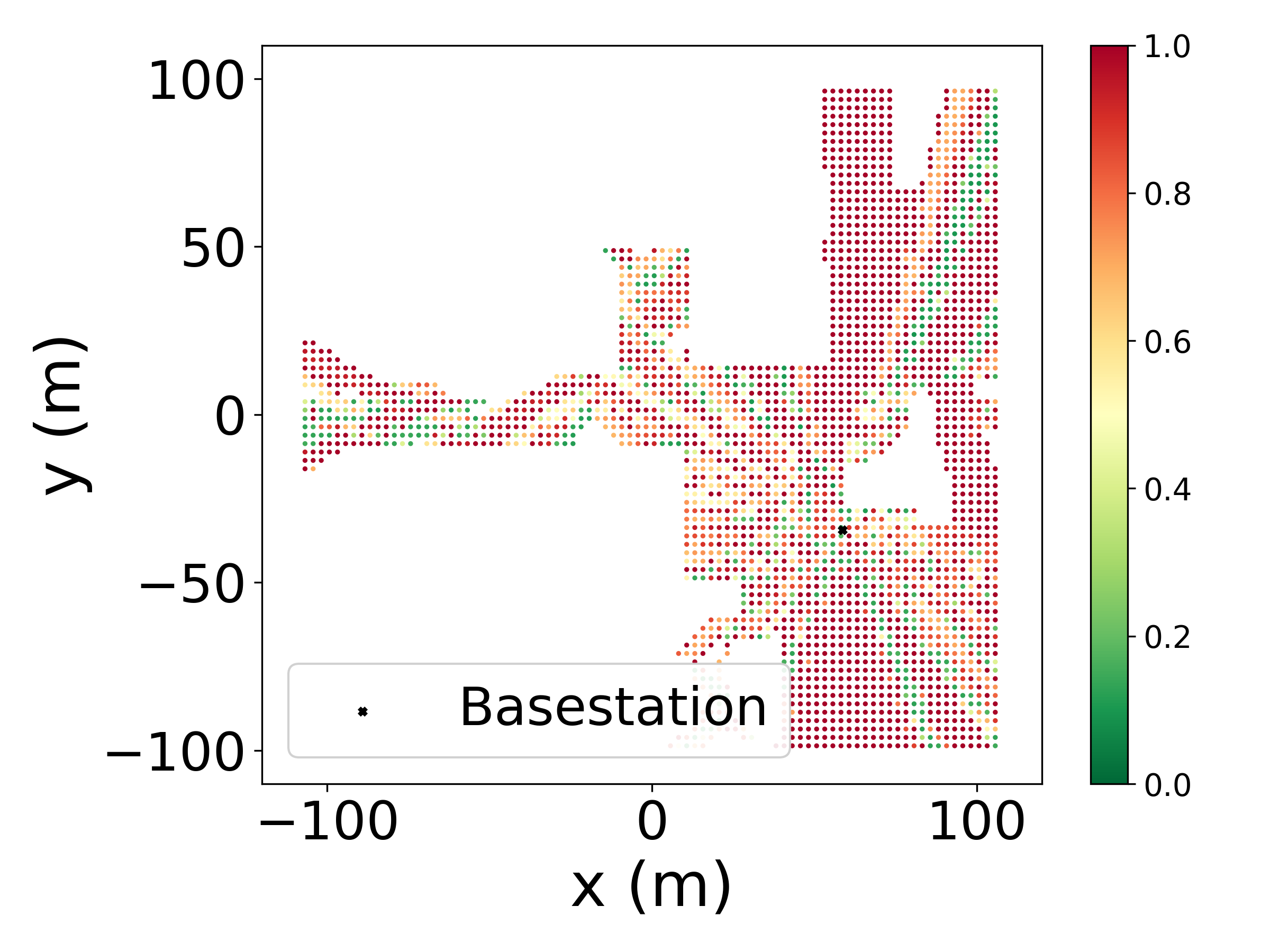}
    \end{minipage}
    \begin{minipage}{0.28\textwidth}
        \includegraphics[width=\linewidth]{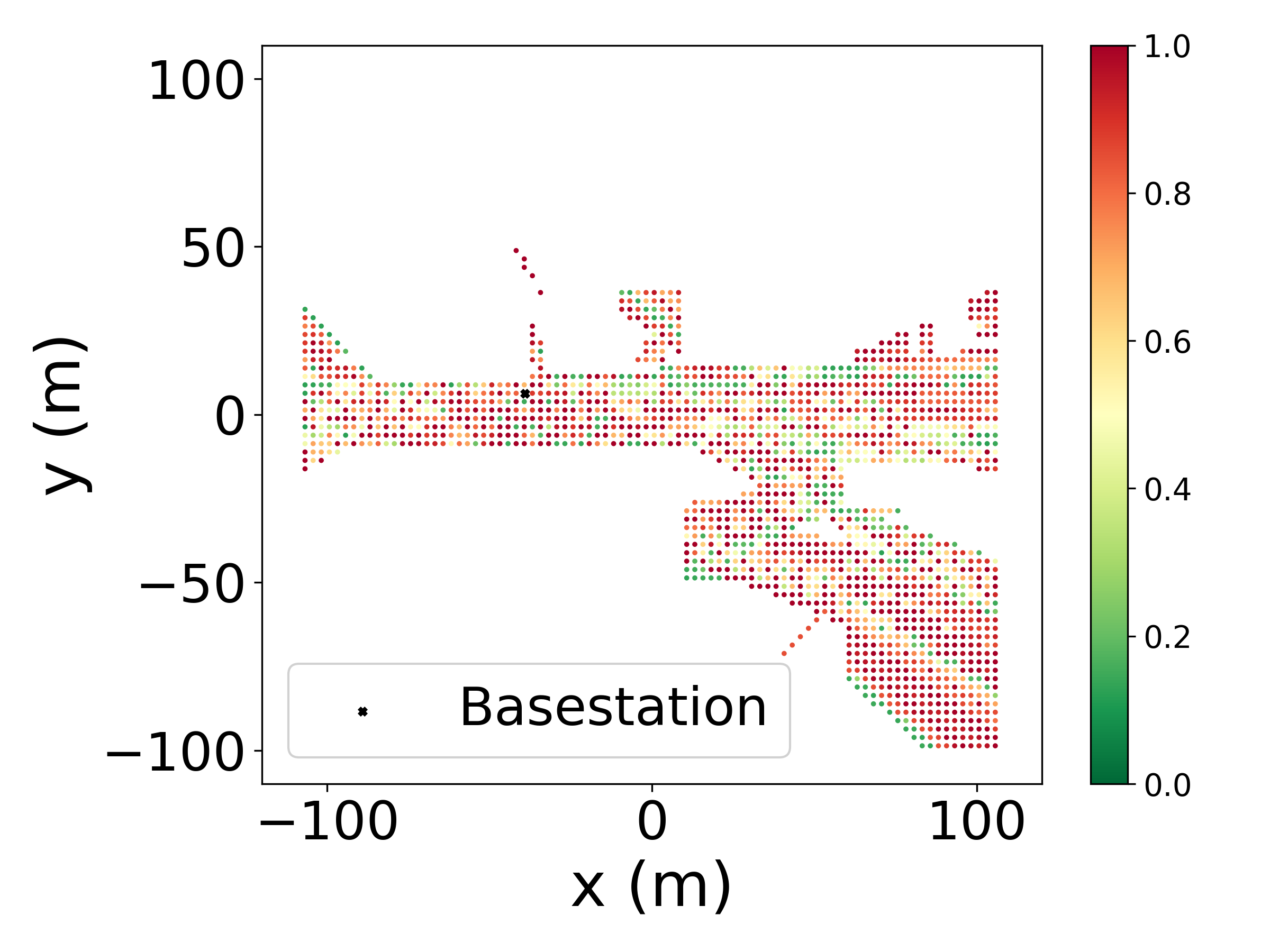}
    \end{minipage}
    \begin{minipage}{0.28\textwidth}
        \includegraphics[width=\linewidth]{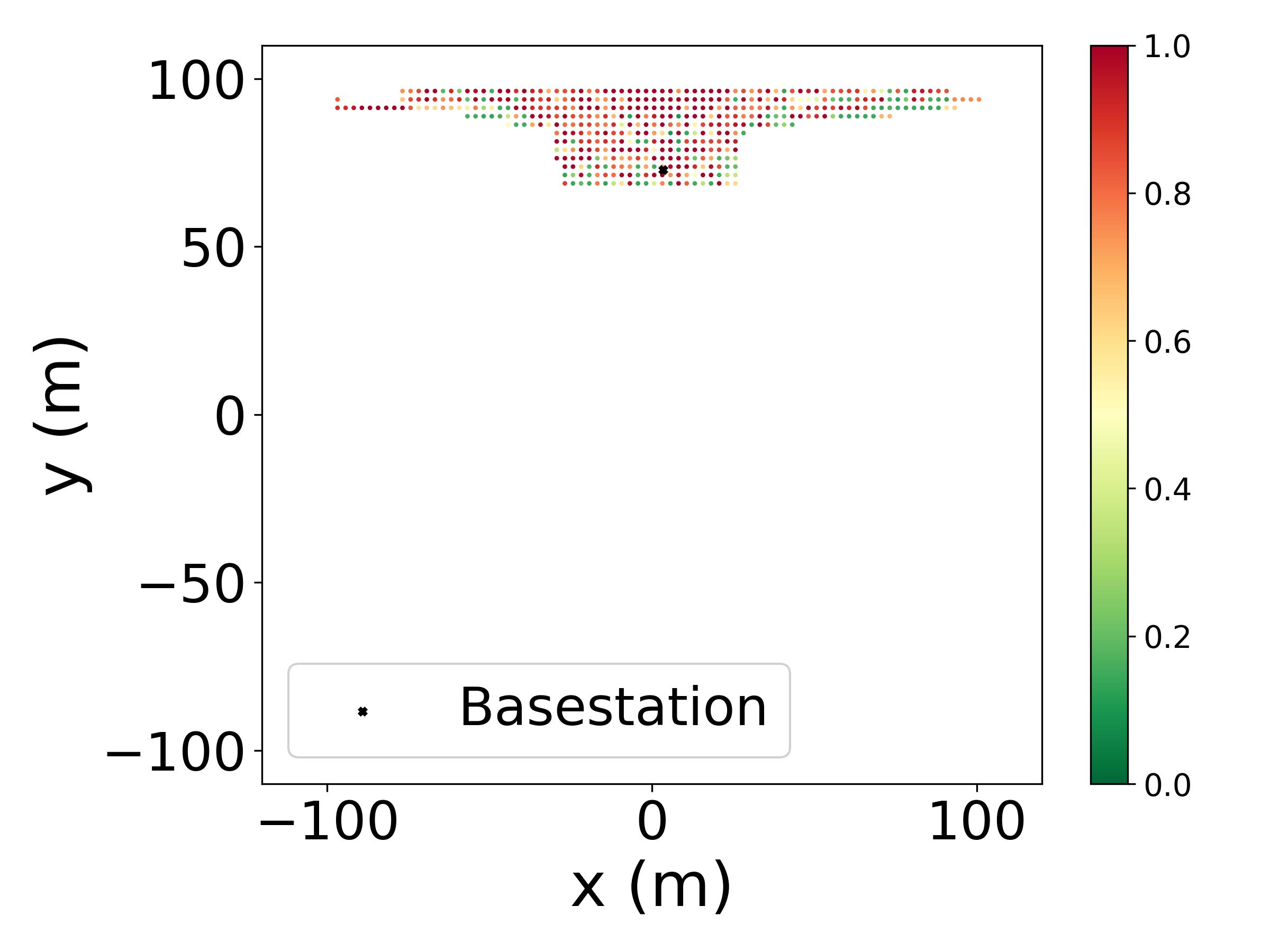}
    \end{minipage}

    \begin{minipage}{0.12\textwidth}
        \raggedright
        \textbf{IMED-MB}
    \end{minipage}
    \begin{minipage}{0.28\textwidth}
        \includegraphics[width=\linewidth]{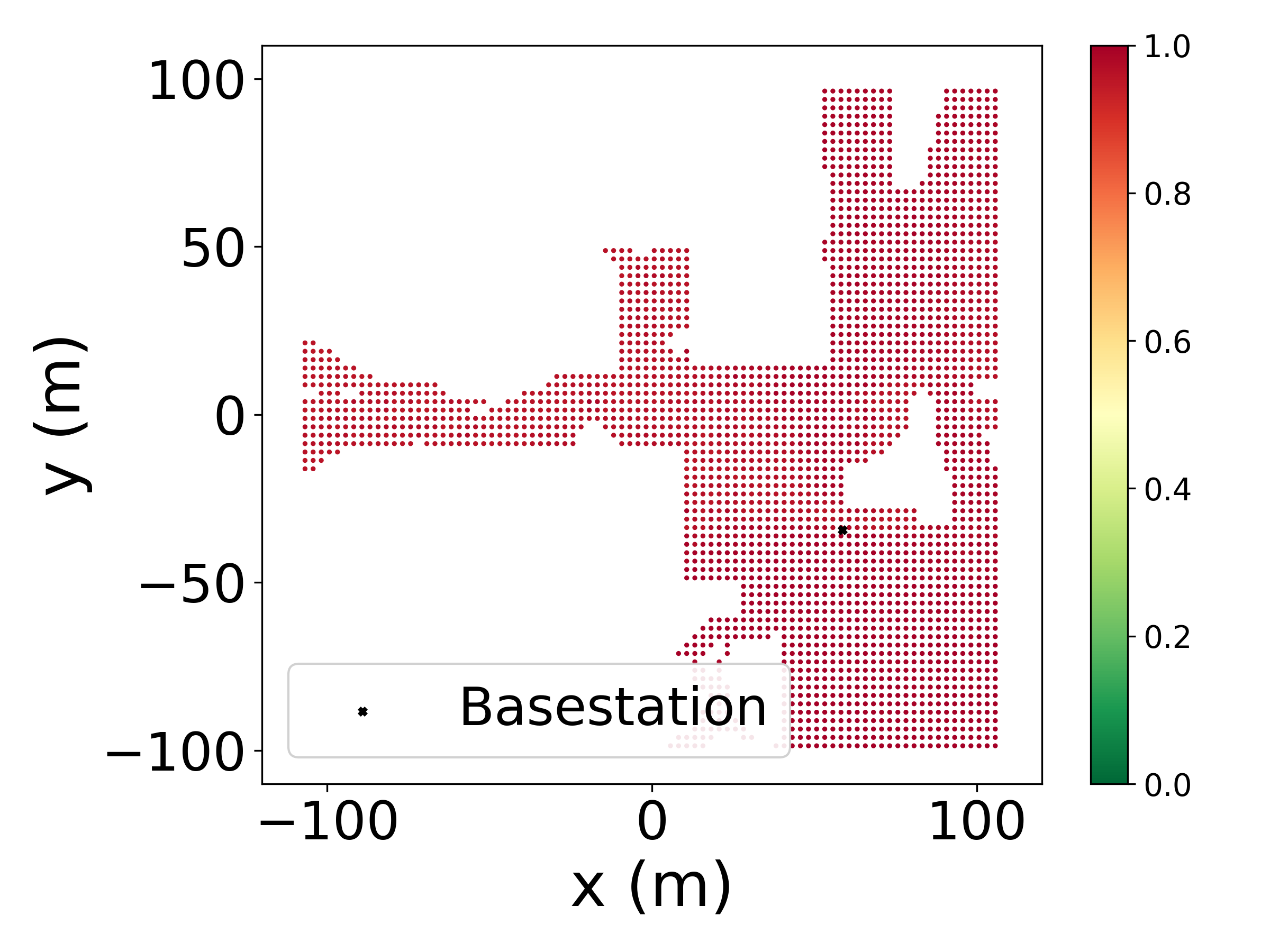}
    \end{minipage}
    \begin{minipage}{0.28\textwidth}
        \includegraphics[width=\linewidth]{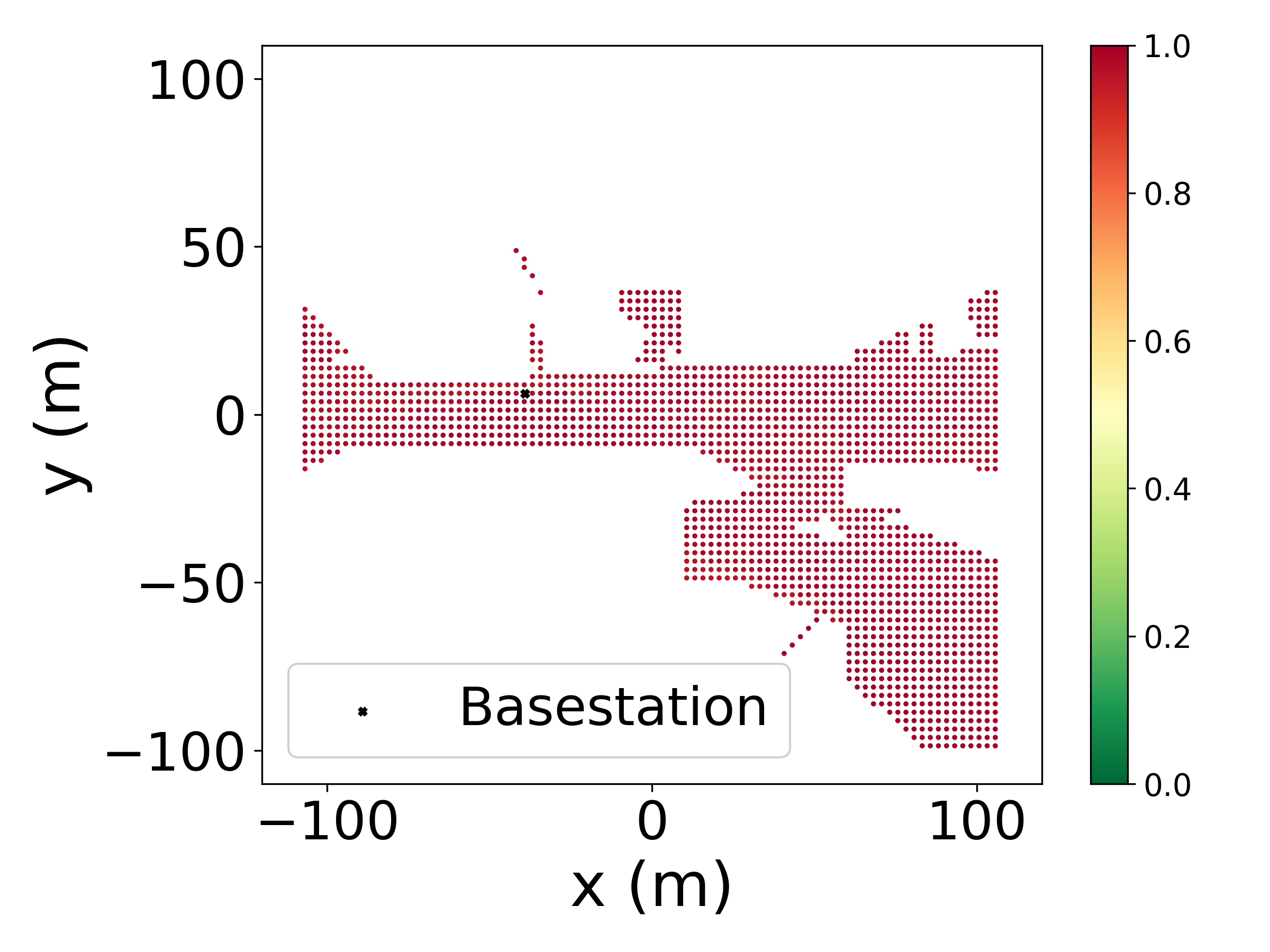}
    \end{minipage}
    \begin{minipage}{0.28\textwidth}
        \includegraphics[width=\linewidth]{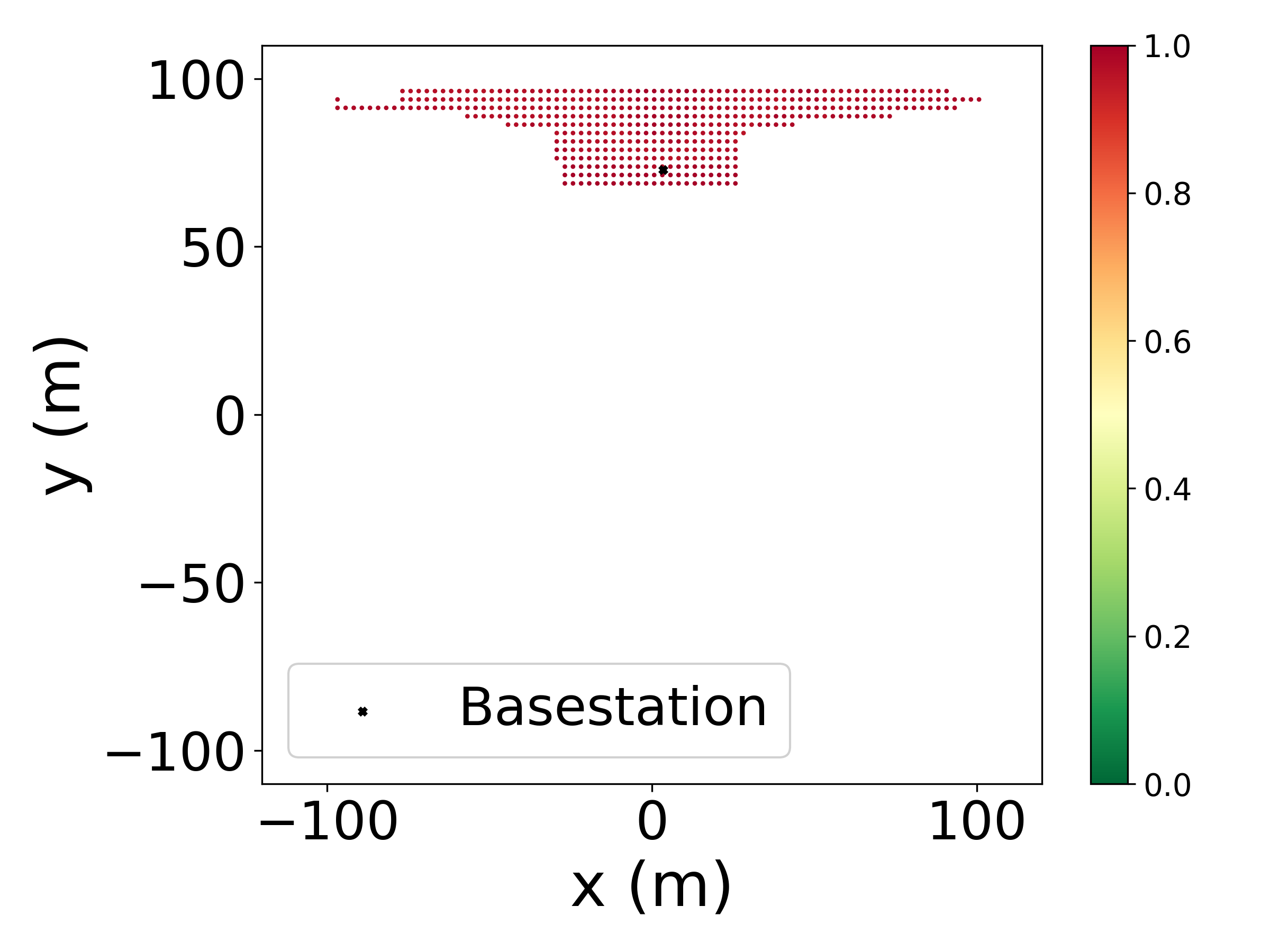}
    \end{minipage}

    \begin{minipage}{0.12\textwidth}
        \raggedright
        \textbf{UCB}
    \end{minipage}
    \begin{minipage}{0.28\textwidth}
        \includegraphics[width=\linewidth]{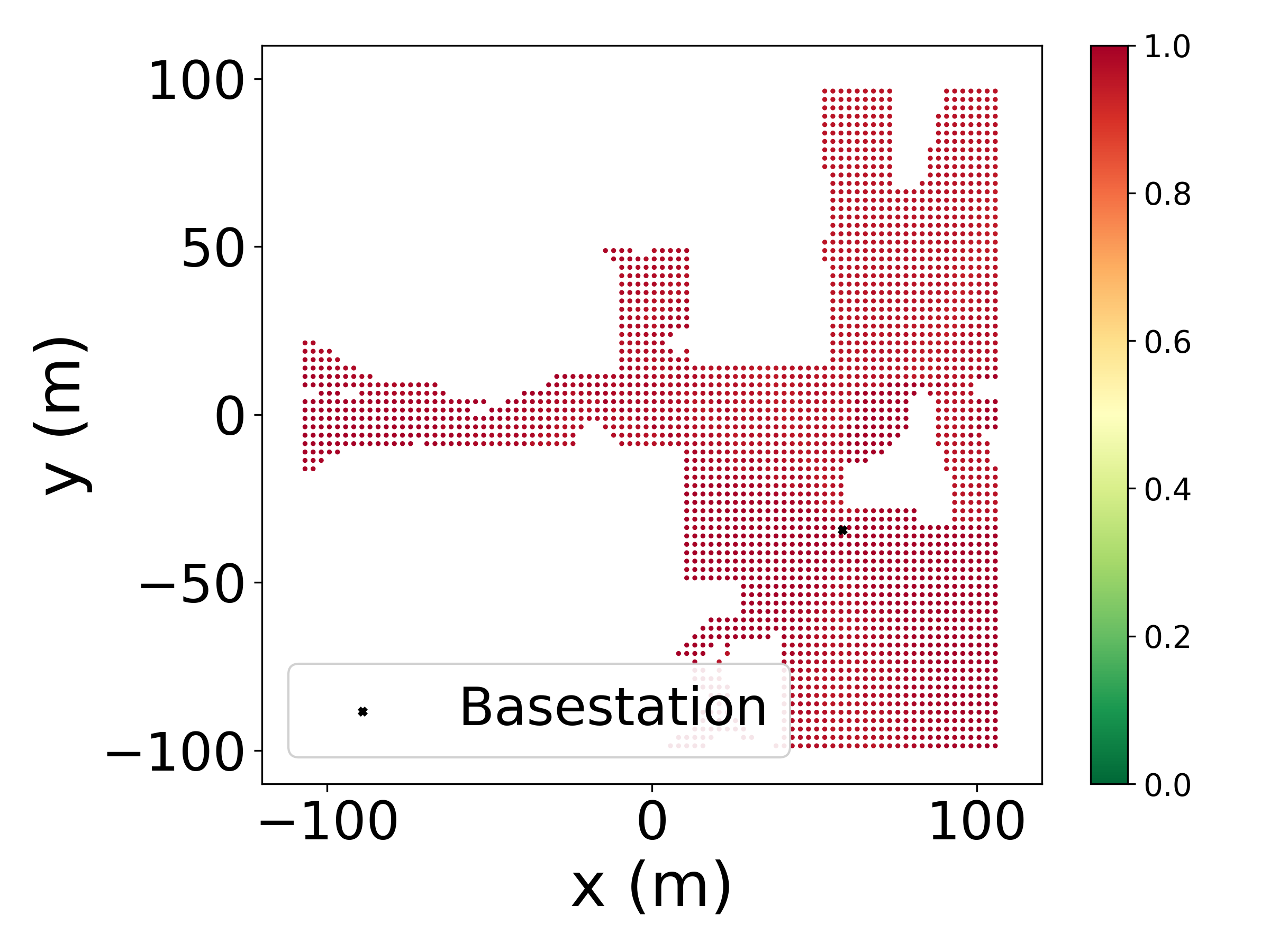}
    \end{minipage}
    \begin{minipage}{0.28\textwidth}
        \includegraphics[width=\linewidth]{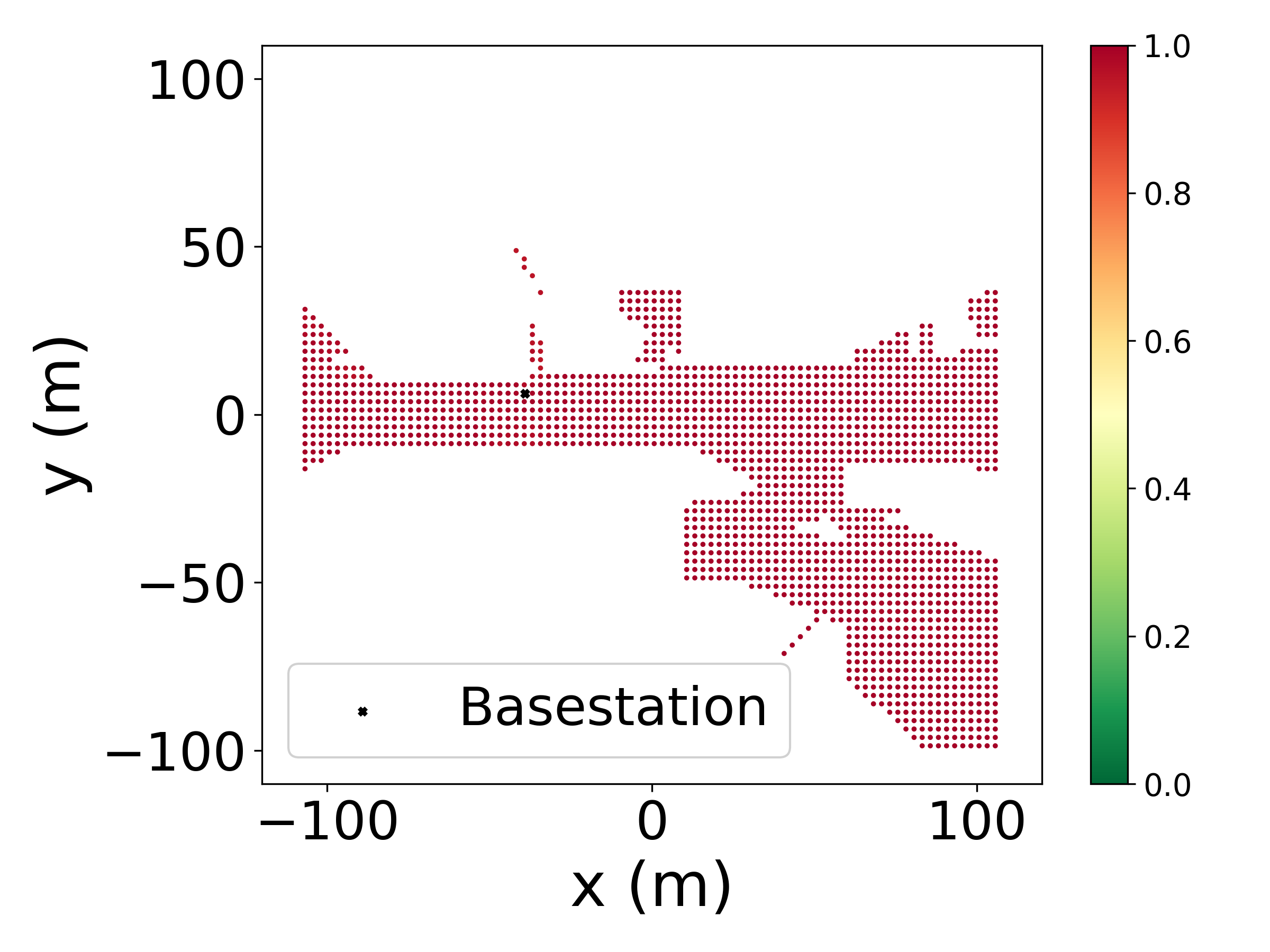}
    \end{minipage}
    \begin{minipage}{0.28\textwidth}
        \includegraphics[width=\linewidth]{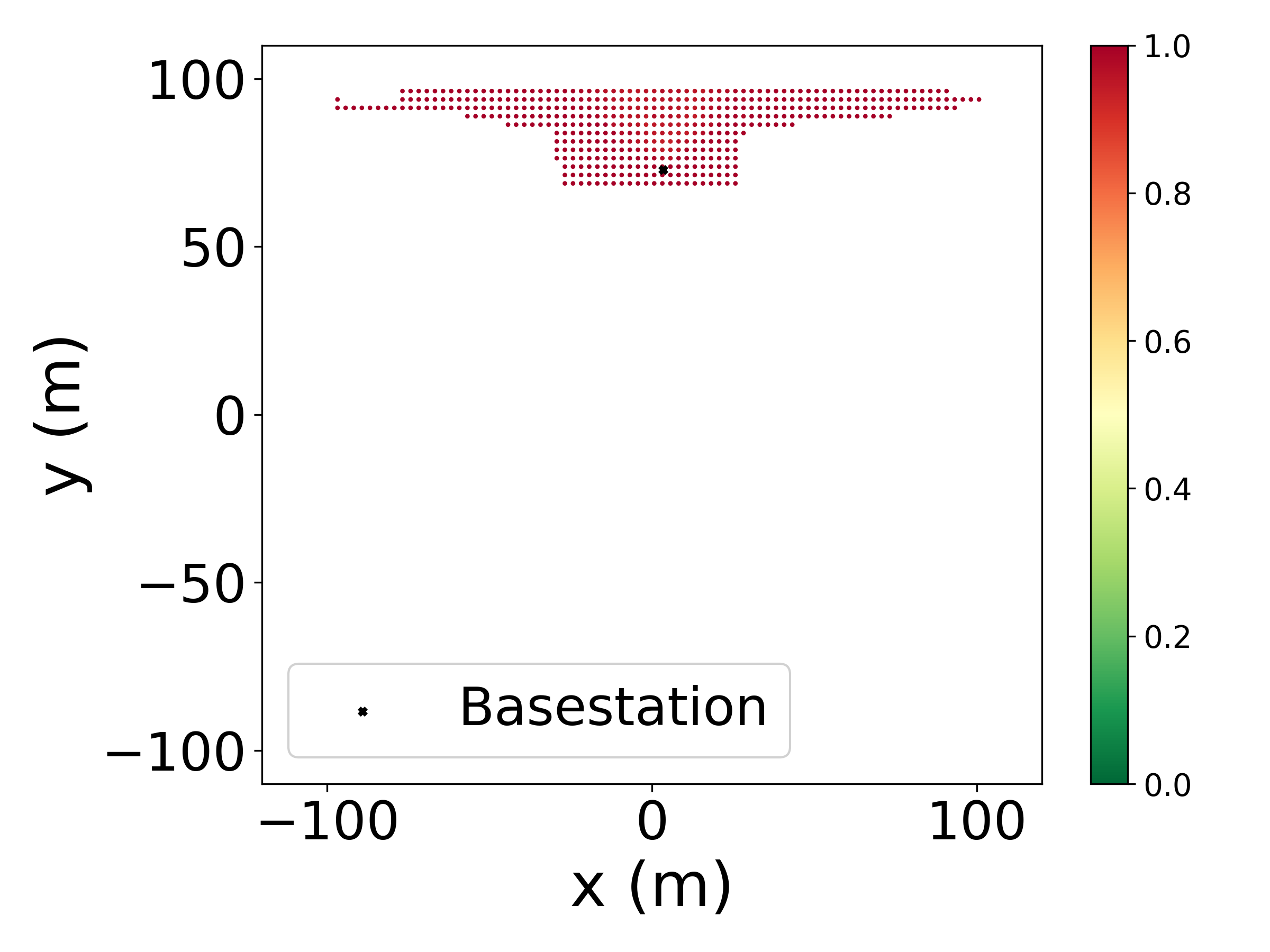}
    \end{minipage}

    \caption{Spatial heat maps of N-Regret$_T$ of different algorithms (rows) across Base Stations 1–3 (columns).
    }
    \label{fig:cumulative-regret-map-matrix}
\end{figure}

\subsection{DeepSense6G Static Setting}

We evaluate the performance of different algorithms on 12 static scenarios from the DeepSense 6G dataset. 
For each scenario, we run each algorithm with $10$ different length-200 intervals sampled randomly from the union of available beam RSS measurement sequences in each scenario, using different random seeds. 

The plots of normalized regret as a function of time step for all 12 scenarios averaged from 10 repetitions are shown in Figure~\ref{fig:static-ave-cum-regret-matrix}.
Recall from Eq.~\eqref{eqn:normalized-static-regret} that, a policy that chooses beams uniformly at random has a normalized regret of 1, while the policy that always chooses the optimal beam has a normalized regret of 0.
An optimal policy will have a normalized regret of 0.
From Figure~\ref{fig:static-ave-cum-regret-matrix}, we observe that \prgreedy with $k=2$ achieves the lowest normalized regret in most scenarios, closely followed by \pretc with $k=2$. \prgreedy with $k=1$ and \pretc with $k=1$ perform slightly worse than their $k=2$ counterparts, but still significantly better than the baseline algorithms, including UCB, BISECTION and IMED-MD. LSE does perform well in some scenarios but is not as good as \prgreedy and \pretc with $k=2$.
This demonstrates that our proposed algorithms consistently achieve lower average normalized regret compared to the baseline algorithms across all scenarios, highlighting their effectiveness in the real-world static settings.

\begin{figure}
    \centering
    \begin{minipage}{0.32\textwidth}
        \includegraphics[width=\linewidth]{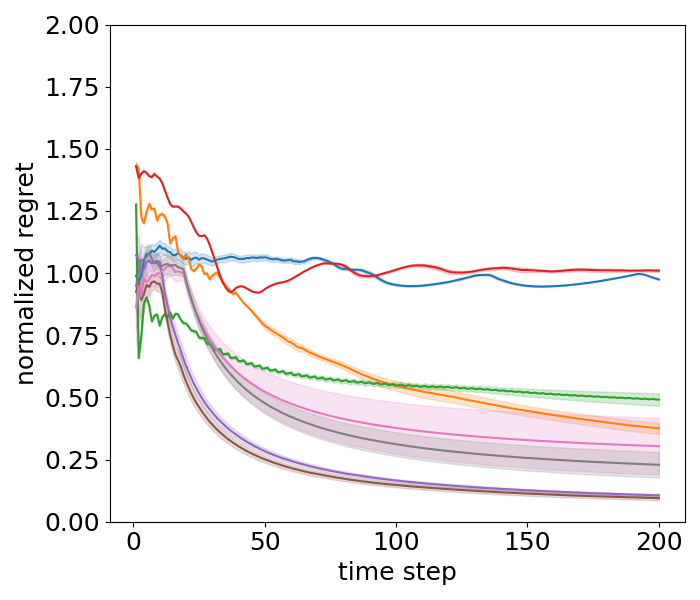}
        \caption*{Scenario17}
    \end{minipage}
    \begin{minipage}{0.32\textwidth}
        \includegraphics[width=\linewidth]{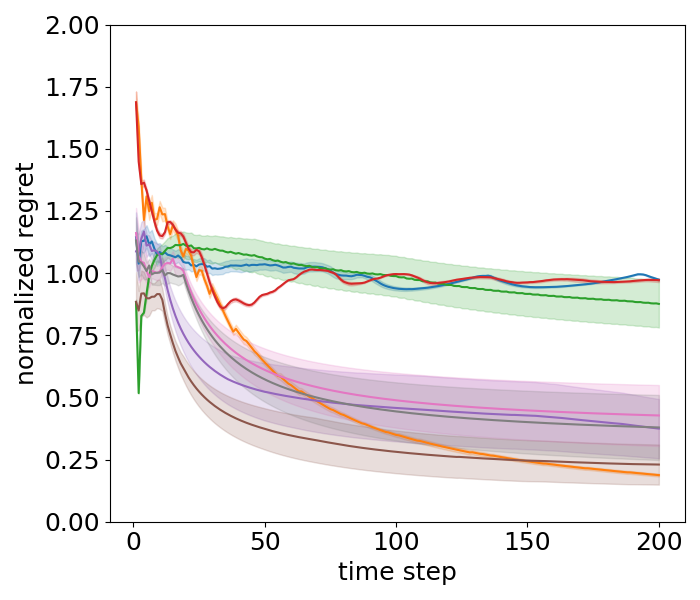}
        \caption*{Scenario18}
    \end{minipage}
    \begin{minipage}{0.32\textwidth}
        \includegraphics[width=\linewidth]{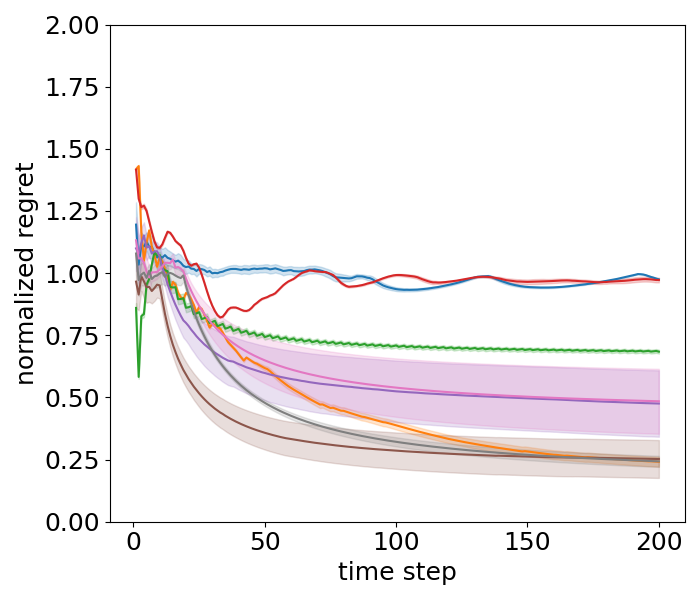}
        \caption*{Scenario19}
    \end{minipage}
    
    \begin{minipage}{0.32\textwidth}
        \includegraphics[width=\linewidth]{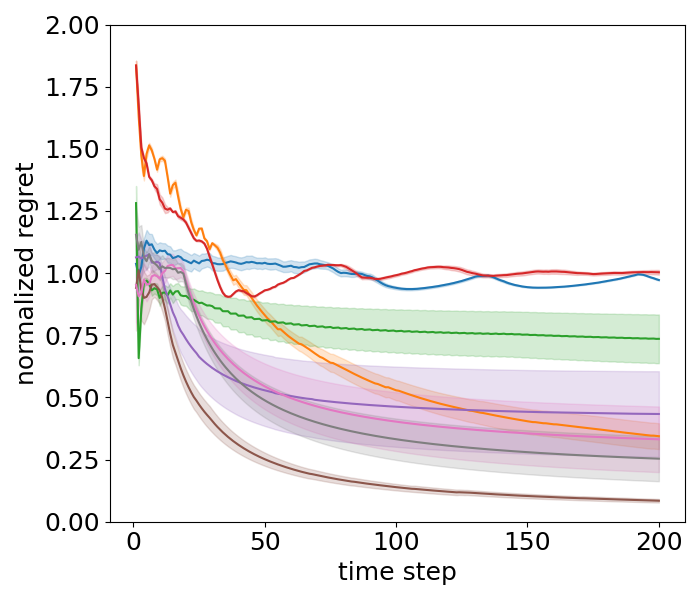}
        \caption*{Scenario20}
    \end{minipage}
    \begin{minipage}{0.32\textwidth}
        \includegraphics[width=\linewidth]{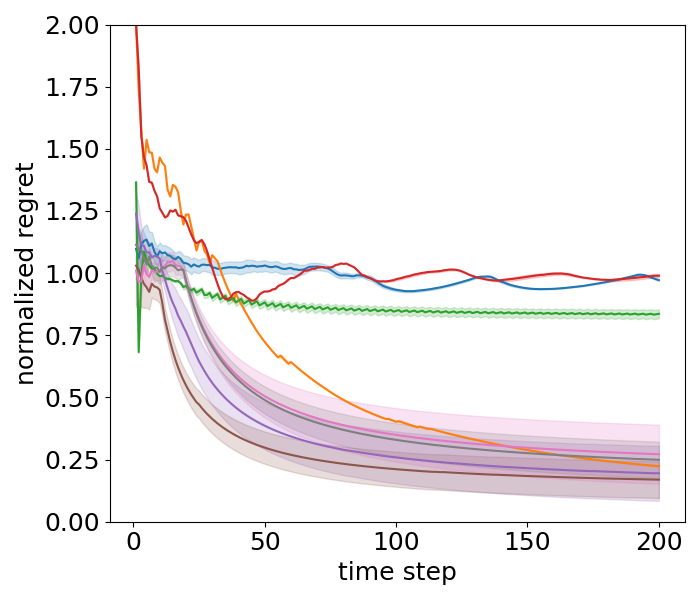}
        \caption*{Scenario21}
    \end{minipage}
    \begin{minipage}{0.32\textwidth}
        \includegraphics[width=\linewidth]{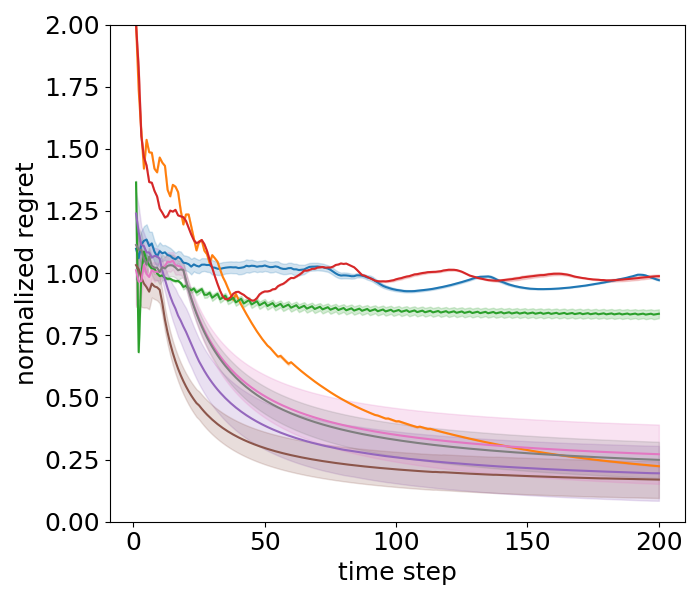}
        \caption*{Scenario22}
    \end{minipage}

    \begin{minipage}{0.32\textwidth}
        \includegraphics[width=\linewidth]{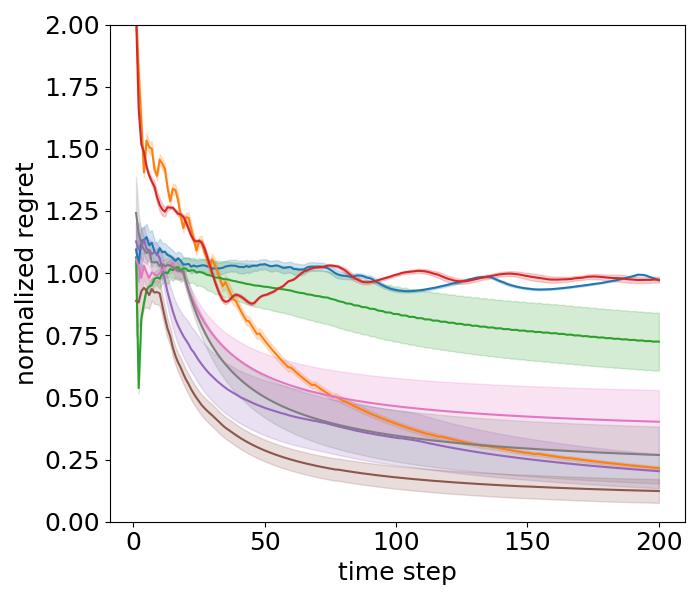}
        \caption*{Scenario24}
    \end{minipage}
    \begin{minipage}{0.32\textwidth}
        \includegraphics[width=\linewidth]{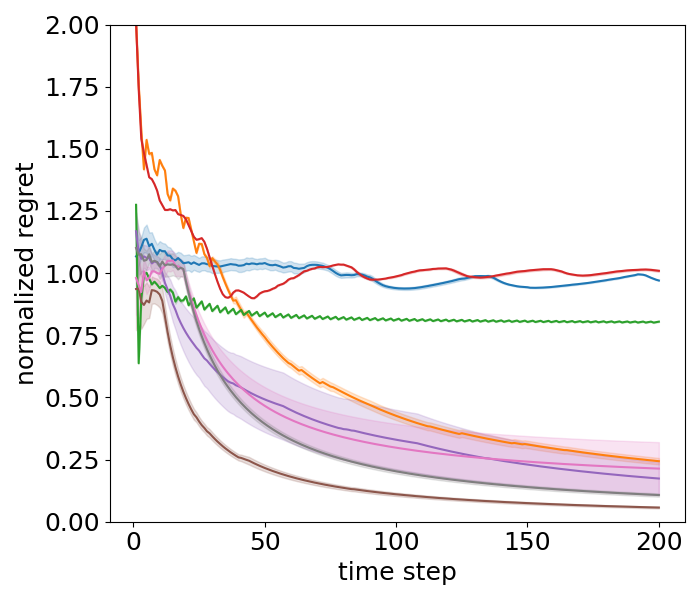}
        \caption*{Scenario25}
    \end{minipage}
    \begin{minipage}{0.32\textwidth}
        \includegraphics[width=\linewidth]{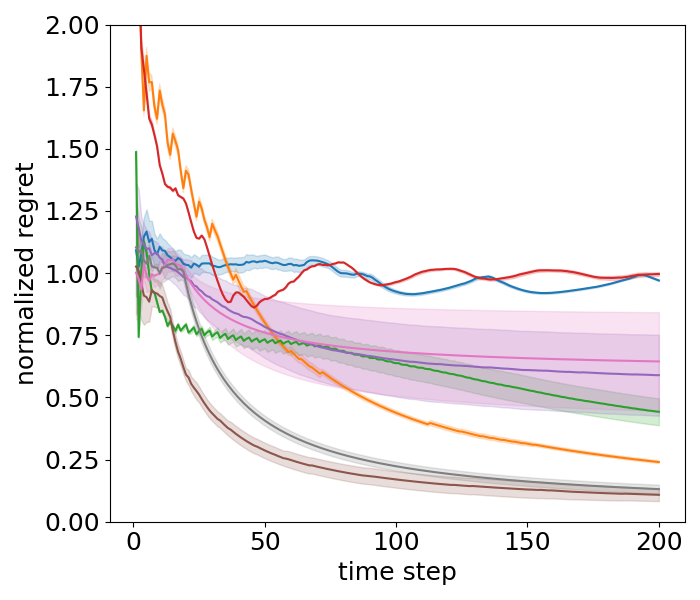}
        \caption*{Scenario26}
    \end{minipage}

    \begin{minipage}{0.32\textwidth}
        \includegraphics[width=\linewidth]{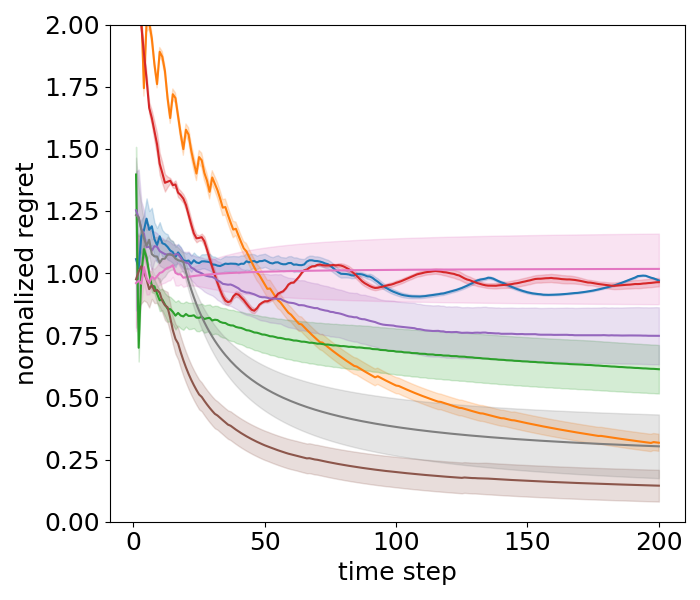}
        \caption*{Scenario27}
    \end{minipage}
    \begin{minipage}{0.32\textwidth}
        \includegraphics[width=\linewidth]{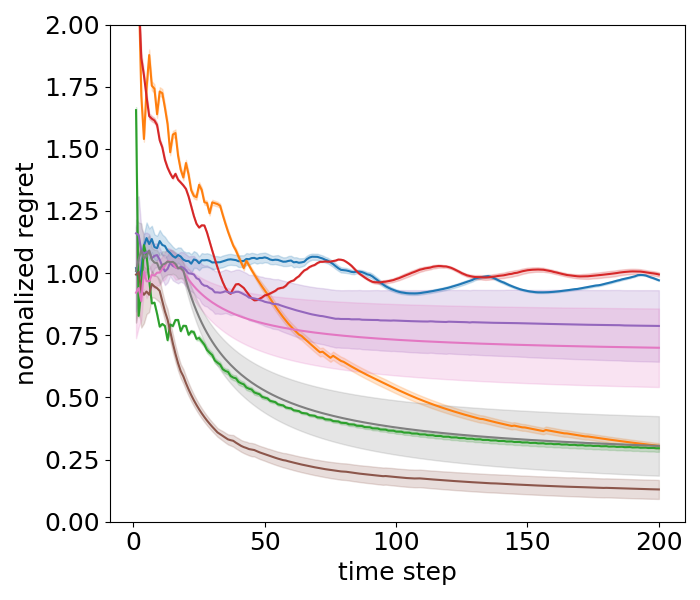}
        \caption*{Scenario28}
    \end{minipage}
    \begin{minipage}{0.32\textwidth}
        \includegraphics[width=\linewidth]{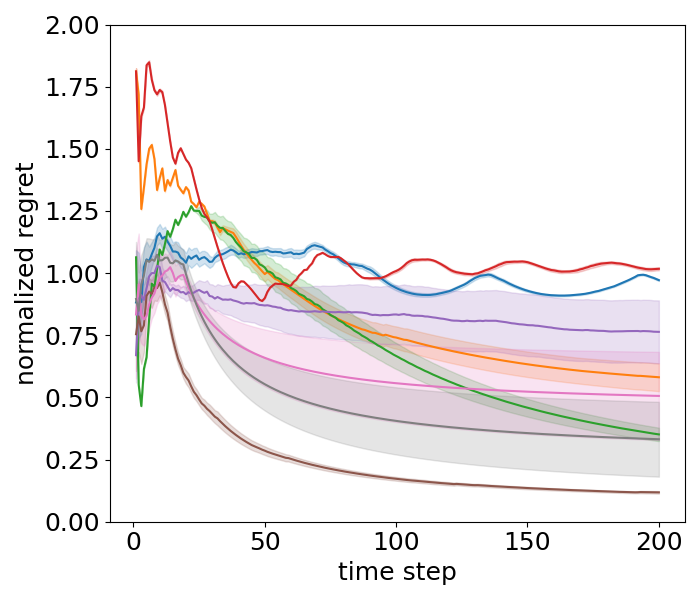}
        \caption*{Scenario29}
    \end{minipage}

    \begin{minipage}{0.98\textwidth}
        \centering
        \includegraphics[width=0.8\linewidth]{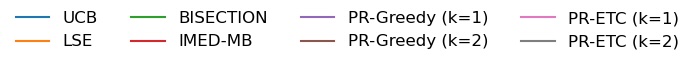}
    \end{minipage}

    \caption{Per step normalized regret plots on 12 static scenarios from the DeepSense 6G dataset. Each plot shows the performance of different algorithms over 200 time steps, averaged over 10 independent trials. The shaded area represents one standard deviation.
    }
    \label{fig:static-ave-cum-regret-matrix}
\end{figure}

\subsection{DeepSense6G Mobile Seting}

We evaluate our proposed algorithms in the dynamic scenario 9 of DeepSense6G~\citep{DeepSense}, which is an environment where the user is moving and the channel state is changing over time.
The original dataset contains 136 different sequences of data, each representing a sequentially recorded RSS information for all 64 beams and each sequence lasts for about 3-5 seconds.
In each sequence, the RSS of each beam is recorded every 0.1 second.
We select 5 different sequences of data to run our experiments, which are indexed 48, 62, 97, 115, and 124 in the dataset.
For each sequence, we interpolate the data to create a semi-synthetic dynamic environment by using a linear interpolation of the RSS's between two consecutive recorded time steps and expanding the number of steps of each sequence by $50$ times.
Therefore, each step in the semi-synthetic data corresponds to $0.1/50 = 0.002$ seconds in the original data.
We run our Periodic-PR-Greedy and Periodic-PR-ETC algorithms by setting the restart step $\tau$ to 50 with $k=1$ and $k=2$, and other baseline algorithms with the same restarting cycle, including UCB, LSE, BISECTION, and IMED-MB.
We repeat each experiment 10 times to evaluate the normalized dynamic regret defined in Eq.~\eqref{eqn:dynamic-normalized-regret} and plot the learning curves in \Cref{fig:normalized-regret-mobile}.

\ifCLASSOPTIONcaptionsoff
  \newpage
\fi

\begin{IEEEbiography}{Hao Qin}
Biography text here.
\end{IEEEbiography}

\begin{IEEEbiographynophoto}{Thang Duong}
Biography text here.
\end{IEEEbiographynophoto}

\begin{IEEEbiographynophoto}{Ming Li}
Biography text here.
\end{IEEEbiographynophoto}

\begin{IEEEbiographynophoto}{Chicheng Zhang}
Biography text here.
\end{IEEEbiographynophoto}

\end{document}